\documentclass[11pt]{article}
\usepackage{fix-cm}
\usepackage[letterpaper,margin=1in]{geometry}
\usepackage[T1]{fontenc}
\usepackage{newpxtext}
\usepackage{amsmath,amssymb,amsfonts,amsthm,mathtools,bm}
\usepackage{newpxmath}
\usepackage{aliascnt}
\usepackage{graphicx,booktabs,placeins}
\usepackage{enumitem}
\usepackage{microtype}
\usepackage[numbers,sort&compress]{natbib}
\usepackage{xcolor}
\definecolor{linkblue}{RGB}{0,92,153}
\usepackage[colorlinks=true,linkcolor=linkblue,citecolor=linkblue,urlcolor=linkblue]{hyperref}
\usepackage[nameinlink,capitalize,noabbrev]{cleveref}
\allowdisplaybreaks
\numberwithin{equation}{section}
\newtheorem{theorem}{Theorem}[section]
\newaliascnt{proposition}{theorem}
\newtheorem{proposition}[proposition]{Proposition}
\aliascntresetthe{proposition}
\newaliascnt{lemma}{theorem}
\newtheorem{lemma}[lemma]{Lemma}
\aliascntresetthe{lemma}
\newaliascnt{corollary}{theorem}
\newtheorem{corollary}[corollary]{Corollary}
\aliascntresetthe{corollary}
\theoremstyle{definition}
\newaliascnt{definition}{theorem}
\newtheorem{definition}[definition]{Definition}
\aliascntresetthe{definition}
\newaliascnt{remark}{theorem}
\newtheorem{remark}[remark]{Remark}
\aliascntresetthe{remark}
\newaliascnt{example}{theorem}
\newtheorem{example}[example]{Example}
\aliascntresetthe{example}
\crefname{theorem}{Theorem}{Theorems}
\crefname{proposition}{Proposition}{Propositions}
\crefname{lemma}{Lemma}{Lemmas}
\crefname{corollary}{Corollary}{Corollaries}
\crefname{definition}{Definition}{Definitions}
\crefname{remark}{Remark}{Remarks}
\crefname{example}{Example}{Examples}

\newcommand{\R}{\mathbb R}
\newcommand{\E}{\mathbb E}
\newcommand{\Prob}{\mathbb P}
\newcommand{\var}{\operatorname{Var}}
\newcommand{\eps}{\varepsilon}
\newcommand{\one}{\mathbf 1}
\newcommand{\cF}{\mathcal F}
\newcommand{\cG}{\mathcal G}
\newcommand{\cE}{\mathcal E}
\newcommand{\cM}{\mathcal M}
\newcommand{\abs}[1]{\left|#1\right|}
\newcommand{\floor}[1]{\left\lfloor#1\right\rfloor}
\newcommand{\ceil}[1]{\left\lceil#1\right\rceil}
\newcommand{\op}{\mathrm{o}_{\Prob}}
\newcommand{\Op}{\mathrm{O}_{\Prob}}
\newcommand{\UCB}{\operatorname{UCB}}
\newcommand{\pto}{\xrightarrow{p}}
\newcommand{\dto}{\Rightarrow}
\DeclareMathOperator*{\argmax}{arg\,max}
\renewcommand{\hat}{\widehat}
\renewcommand{\tilde}{\widetilde}
\title{Allocation Stability and Wald Inference\\under Variance-Aware UCB}
\author{Yingying Fan \and Yuxuan Han \and Jinchi Lv \and Xiaocong Xu \and Zhengyuan Zhou}

\begin{document}
\maketitle
\begin{abstract}
Allocation stability is often used to justify Gaussian inference from bandit data, but when is it necessary? In this paper, we address this question for a two-armed, fixed-horizon variance-aware UCB policy with bounded reward distributions that may vary with the horizon. We find a sharp criterion in terms of the reward gap and variances that determines whether the optimal-arm count admits a deterministic approximation with vanishing relative error, while the suboptimal-arm count is always stable. Despite the possible instability of the optimal-arm count, we show that the ordinary Wald statistic for a linear combination of the arm means has a standard normal limit for every fixed nonzero coefficient vector, provided the product of the pull count and reward variance diverges in probability for each arm. Under the same condition, however, this Gaussian approximation holds uniformly over deterministic nonzero coefficient vectors if and only if the optimal-arm count is stable. The analysis relies on two main ingredients: (i) a pathwise comparison with an auxiliary policy whose final optimal-arm count is asymptotically equivalent to the original count and independent of the optimal-arm reward sequence; and (ii) joint limits for the rescaled optimal-arm count and the two studentized sample-mean errors under the original policy, which yield nonstandard Wald limits for certain linear combinations of the arm means with coefficients that vary with the horizon.
\end{abstract}

\section{Introduction}\label{sec:intro}

Suppose we use a bandit algorithm to collect data and then want to estimate the arm means or their difference and construct confidence intervals. We would usually use the sample means, with a separate estimated standard error for each arm. The difficulty is that the number of observations from each arm depends on the rewards already seen. This dependence can change the distribution of the sample means \citep{nie2018adaptively,zhang2020inference}, so we need to justify the Gaussian approximation behind the usual Wald procedure.

Stability gives one way to justify this approximation. Here, stability means that each arm's count can be approximated by a deterministic count with vanishing relative error (see \cref{def:stable}). Under suitable moment and sample-size conditions, we can use these deterministic counts to normalize the centered reward sums and apply a martingale central limit theorem. This use of deterministic design approximations for inference goes back to \citet{lai1982least}. For canonical upper confidence bound (UCB) algorithms, \citet{kalvit2021closer} study limiting allocation proportions, while \citet{khamaru2024inference} establish relative count approximations for every arm. \citet{han2024ucb} give quantitative bounds relating these count errors to Gaussian approximation errors.

However, these arguments show only that stability is sufficient under suitable conditions, but do not tell us whether it is necessary. Since every arm's count is stable under canonical UCB, this algorithm does not let us examine what happens when stability fails. In this paper, we consider a variance-aware UCB rule. Whereas canonical Hoeffding-based UCB uses a common bound on the reward range to choose its exploration bonus, variance-aware rules use arm-specific variance estimates \citep{auer2002finite,audibert2006use,audibert2009exploration}. Figure~\ref{fig:optimal-count-distribution} compares the distributions of the optimal-arm count under the two algorithms in a simple two-arm instance, showing that the optimal-arm count is concentrated under canonical UCB but much more dispersed under variance-aware UCB. This suggests that, under variance-aware UCB, the optimal-arm count may fail to admit a deterministic approximation.

\begin{figure}[!t]
    \centering
    \includegraphics[width=0.72\textwidth]{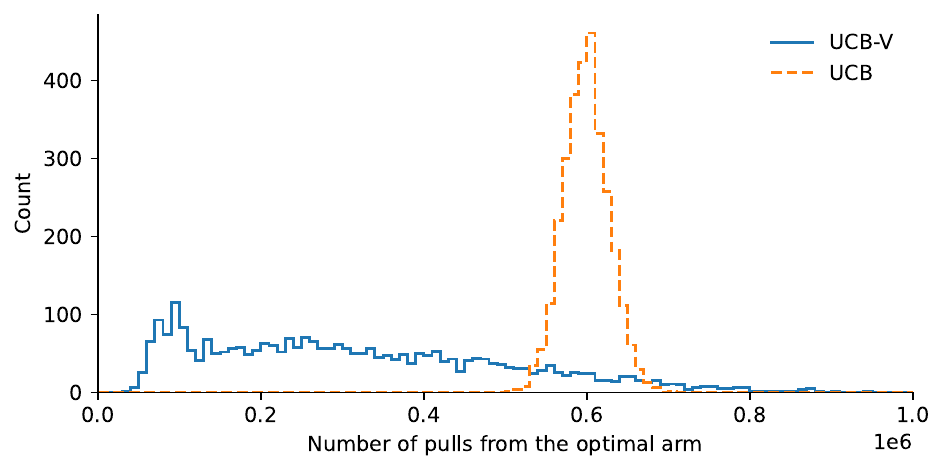}
    \caption[Distribution of the optimal-arm count.] {Histograms of the number of pulls from the optimal arm over 3,000 independent replications for each algorithm, with $T=10^6$ rounds. Both algorithms use the exploration level $c_T=\rho\log T$ with $\rho=2$. Arm 1 yields $0.5+0.2\sqrt{c_T/T}\pm0.2c_T^{-3/4}$ with equal probability, while arm 2 yields $0.3$ or $0.7$ with equal probability.}
    \label{fig:optimal-count-distribution}
\end{figure}

\begin{samepage}
We therefore focus on two closely related questions about variance-aware UCB:
\begin{enumerate}[label=(\roman*),leftmargin=2.5em,itemsep=0.2em,topsep=0.4em]
    \item \emph{Stability:} When do the final pull counts admit deterministic approximations with vanishing relative error?
    \item \emph{Inference:} When does the usual Wald statistic for a linear combination of the arm means have a Gaussian limit, and is count stability necessary?
\end{enumerate}
\end{samepage}

\subsection{Model and policy}\label{sec:model-policy}

To make these questions precise, we consider a two-armed bandit for each horizon $T\geq3$. For each arm $a\in\{1,2\}$, let $X_{a,1},X_{a,2},\ldots$ be an i.i.d.\ reward sequence supported on $[0,1]$, and assume that the two sequences are independent. Write 
\begin{align*}
    \mu_a=\E[X_{a,1}],\qquad \sigma_a^2=\var(X_{a,1}), \qquad \Delta=\mu_1-\mu_2>0
\end{align*}
Throughout, we use the convention that arm 1 denotes the optimal arm and arm 2 the suboptimal arm.

For the asymptotic analysis, the reward laws may depend on $T$, so the gap and the two variances may vary with the horizon. This allows local regimes in which the gap approaches a variance-dependent scale as $T$ grows, in the same spirit as small-gap and weak-signal asymptotics \citep{kalvit2021closer,kuang2023weak,chen2025adaptivity}. We suppress the dependence on $T$ and assume $\sigma_2>0$; arm 1 may be deterministic.

We index observations from each arm in pull order. For $n\geq1$, define
\begin{align}
  \bar X_{a,n}:=\frac1n\sum_{i=1}^nX_{a,i},\qquad \hat\sigma_{a,n}^2:=\frac1n\sum_{i=1}^n\bigl(X_{a,i}-\bar X_{a,n}\bigr)^2.\label{eq:sample-mean}
\end{align}
If $A_t\in\{1,2\}$ is the arm selected at round $t$, let $n_{a,t}:=\sum_{s=1}^t\one\{A_s=a\}$, with $n_{a,0}:=0$, and write $n_{a,T}$ for the final pull count. The policy first pulls each arm once in a fixed order. For an arm with $n\geq1$ observations, define
\begin{align*}
  \hat B_a(n) := \Big( \hat\sigma_{a,n}\sqrt{\frac{\rho\log T}{n}} \Big) \vee \frac{\rho\log T}{n}, \qquad \UCB_a(n):=\bar X_{a,n}+\hat B_a(n),
\end{align*}
where $\rho>1$ is fixed and $x\vee y:=\max\{x,y\}$. At each subsequent round $t\geq3$, the policy selects an arm in $\argmax_{a\in\{1,2\}}\UCB_a(n_{a,t-1})$; ties are resolved by a fixed rule. The two terms in the exploration bonus play different roles. The standard-deviation term adjusts the bonus to the arm's variability, while the inverse-count term keeps it from becoming too small when the empirical standard deviation is small, especially after only a few pulls. We use a fixed-horizon version of the policy: the horizon is known in advance, and the same factor $\rho\log T$ is used at every round. For brevity, write $c_T:=\rho\log T$.

\subsection{Overview of the results}\label{sec:overview}
We first recall the notion of stability, which goes back to the work of \citet{lai1982least} on stochastic regression. For bandit data, we use the following arm-wise formulation \citep[Definition~2.1]{khamaru2024inference}.

\begin{definition}\label{def:stable}
Arm $a$ is \emph{stable} if there is a positive deterministic sequence $m_{a,T}\to\infty$ such that $n_{a,T}/m_{a,T}\pto1$. The normalizer may depend on the reward laws.
\end{definition}
In this definition, the approximation error must vanish relative to $m_{a,T}$, which may be much smaller than $T$. Thus an arm can be stable even when it receives a vanishing fraction of the total budget.

\paragraph{Allocation stability.}
For canonical UCB, all arm counts are stable \citep[Theorem~3.1]{khamaru2024inference}. Under the variance-aware rule studied here, the suboptimal-arm count is always stable, whereas the optimal-arm count need not be. We show in \cref{thm:sharp} that the optimal-arm count is unstable \textit{if and only if} the following three conditions hold simultaneously along some subsequence:
\begin{align}
 \left|\Delta-\sigma_2\sqrt{\frac{c_T}{T}}\right| =O\left(\frac{\sigma_2}{\sqrt T}\right),\qquad \frac{\sigma_1}{\sigma_2}=O(c_T^{-3/4}),\qquad \frac{c_T^3}{T\sigma_2^2}=O(1). \label{eq:overview-critical-window}
\end{align}
These conditions describe a critical window in which fluctuations of the arm-2 sample mean may affect the optimal-arm allocation at first order. The scale comparison underlying this critical window is made precise in \cref{sec:local-separation}.

To describe the allocation within this window, we show in \cref{prop:critical-count} that every subsequence of horizons satisfying \eqref{eq:overview-critical-window} has a further subsequence along which $n_{1,T}/(T/\sqrt{c_T})$ converges in distribution to a nonconstant random variable.

\paragraph{Gaussian inference.}\label{sec:inference-overview}
We next consider inference for a linear combination of the arm means, $\theta_T(u)=u_1\mu_1+u_2\mu_2$, where $u=(u_1,u_2)\neq0$ is deterministic. Let $W_T(u)$ be the usual Wald statistic based on the sample means and their estimated standard errors at the final pull counts; see \eqref{eq:linear-wald} for the precise definition.

Suppose both effective sample sizes diverge, i.e., $n_{a,T}\sigma_a^2\pto\infty$ for $a=1,2$. We show in \cref{thm:fixed-target} that
\begin{align*}
 W_T(u)\dto N(0,1) \qquad\text{for every fixed }u\in\R^2\setminus\{0\},
\end{align*}
even when the optimal-arm count is unstable. Thus the usual Wald intervals for either arm mean and for their difference have asymptotically correct coverage. 

Our main technical step is a pathwise comparison with an auxiliary policy that replaces the optimal arm's empirical index by its population counterpart. We show that this policy's final optimal-arm count is independent of the entire arm-1 reward sequence and approximates $n_{1,T}$ with vanishing relative error (\cref{prop:cross-arm}). Note that this count may remain random, but revealing the arm-2 rewards at time zero makes it measurable at time zero, so it can be used as a random normalizer in a martingale central limit theorem. This then yields Gaussian inference for the optimal-arm mean under the effective-sample-size condition.

The result above is pointwise in the coefficient vector $u$: the coefficients are fixed as $T$ grows. A stronger question is whether the same Gaussian approximation continues to hold when the coefficient vector is allowed to vary with the horizon. For a fixed sequence of reward laws, this amounts to asking whether $W_T(u_T)\dto N(0,1)$ for every deterministic sequence $u_T\neq0$, or equivalently whether
\begin{align*}
\sup_{u\in\R^2\setminus\{0\}}\sup_{z\in\R} \left|\Prob\{W_T(u)\leq z\}-\Phi(z)\right| \longrightarrow0,
\end{align*}
where $\Phi$ is the standard normal distribution function. Under the same effective-sample-size condition, we show in \cref{thm:uniform-wald}  that this uniform Gaussian approximation holds if and only if the optimal-arm count is stable. When stability fails, we further identify in \cref{cor:critical-coefficients} the coefficient scales that lead to standard and nonstandard Wald limits.

\subsection{Related work}\label{sec:related}

\paragraph{Variance-dependent exploration.}
Early work on variance-aware exploration in multi-armed bandits includes \citet{auer2002finite} and \citet{audibert2006use,audibert2009exploration}, with the latter proposing and analyzing the UCB-V algorithm. Later work considers elimination-based algorithms that use variance information \citep{mukherjee2018efficient} and Thompson sampling in variance-aware settings \citep{honda2014optimality,saha2024only}. We study a fixed-horizon rule that uses the larger of the two bonus terms rather than their sum, with $\log T$ held fixed throughout the experiment.

Related methods use variance information in contextual bandits and Markov decision processes \citep{talebi2018variance,zhang2021improved,dai2022variance, zhao2023variance,di2023variance,xu2024noise}. In particular, \citet{zhang2021improved,zhao2023variance} obtain regret bounds involving the cumulative round-wise noise variances $\sigma_t^2$. When these variances are all equal to $\sigma^2$, the leading variance-dependent terms scale with $\sigma\sqrt{T}$, up to dimension-dependent and logarithmic factors. Here we study how the individual arm variances affect the pull counts and use this analysis to compare regret and Wald interval lengths under variance-aware and canonical UCB; see \cref{sec:comparison}.

\paragraph{Asymptotic behavior in multi-armed bandits.}
Recent work has characterized the asymptotic behavior of pull counts under canonical UCB \citep{kalvit2021closer,khamaru2024inference,han2024ucb}. \Citet{chen2025adaptivity} further derive a joint central limit theorem for the pull counts and sample means, including second-order behavior in small-gap regimes. Diffusion-limit analyses study bandit algorithms with reward gaps of order $1/\sqrt{T}$ \citep{fan2021diffusion,kuang2023weak}. Related diffusion approximations for sequential experiments use a high-frequency, low-information scaling \citep{araman2022diffusion}. More recently, \citet{han2026thompson} obtains precise pull-count limits and corresponding inferential results for Thompson sampling. Other work studies regret distributions and their asymptotic behavior \citep{fan2021fragility,fan2022typical,simchi2023regret}, including in the fixed-gap regime of \citet{lai1985asymptotically}, where $\Delta$ remains constant as $T$ increases.

Recent work also shows that stability can fail. \Citet{praharaj2025instability} prove such a failure for a family of minimax-optimal optimism-based algorithms, and \citet{chen2026allocational} establish a worst-case trade-off between regret and allocation variability, measured by the largest standard deviation of an arm's count. Those papers consider broad algorithm classes or worst-case instances. Here we fix the policy, allow the instance parameters to vary with $T$, and identify a critical window directly in terms of these parameters.

\paragraph{Inference after adaptive sampling.}
Online debiasing and adaptive inference methods have been studied by \citet{dimakopoulou2021online,chen2021statistical,chen2022online,duan2024online}. Our results are more closely related to post-policy inference based on data collected from prespecified adaptive policies \citep{nie2018adaptively,zhang2020inference,hadad2021confidence,zhan2021off}. Confidence sequences provide alternatives to Gaussian inference that do not require pull-count stability \citep{howard2021timeuniform}; see also \citet{johari2022always} for always-valid inference in continuously monitored A/B tests. For several common adaptive designs, \citet{cho2025simulation} develop simulation-based inference that can remain valid when pull counts are unstable. We refer to \citet{bibaut2025demystifying} for a review.

Stability-based inference extends beyond finite-armed bandits. Design stability supports inference for treatment effects under sequentially adaptive assignment \citep{sengupta2026design}, while target-specific directional stability can suffice even when the full design does not stabilize \citep{shen2026directional}. That approach relies on a target-specific predictable quadratic variation that stabilizes. In the enlarged filtration, normalization using the auxiliary count gives predictable quadratic variation converging to one even when the optimal-arm count is unstable (\cref{lem:cross-arm-gaussian}). Related results have been developed for linear contextual bandits and LinUCB \citep{praharaj2026contextual,fan2025linucb}. Beyond first-order Gaussian limits, \citet{wang2026bias} show that index-based policies with stable pull counts can still induce a slowly vanishing bias in the corresponding $Z$ statistics and derive its leading term.

Other methods obtain stability by modifying the sampling rule. Variance inflation and optimism stabilize Thompson sampling \citep{halder2025stable,yan2026optimism}, while regularization yields stable bandit algorithms with quantitative central limit theorems \citep{halder2026regularization}. We instead leave the policy unchanged and study ordinary Wald inference using all the observations it collects.

\paragraph{Asymptotic notation.}
For positive sequences $f_T,f'_T$, write $f_T\lesssim f'_T$ if $f_T\leq Cf'_T$ for all sufficiently large $T$, and write $f_T\asymp f'_T$ if both $f_T\lesssim f'_T$ and $f'_T\lesssim f_T$. A subscript in $\lesssim_\eps$ or $\asymp_\eps$ allows the constant to depend on $\eps$. For random variables, $\zeta_T=\op(f_T)$ means $\zeta_T/f_T\xrightarrow{p}0$, and $\zeta_T=\Op(f_T)$ means that $\zeta_T/f_T$ is bounded in probability. All limits are taken as $T\to\infty$. Throughout, a subsequence means a subsequence of horizons $T\to\infty$; when we restrict attention to a subsequence, subsequent asymptotic statements are understood along that subsequence unless stated otherwise. A constant without a subscript may change from line to line.

\paragraph{Organization.}
The paper is organized as follows. In \cref{sec:stability}, we establish the sharp stability boundary and derive the critical count limit. In \cref{sec:inference}, we first explain how stability gives Gaussian inference, then study fixed linear targets and uniform approximation over target coefficients. In \cref{sec:comparison}, we compare regret and Wald interval lengths under variance-aware and canonical UCB. \Cref{sec:numerics} gives a numerical illustration of the inference results, and \cref{sec:discussion} discusses the remaining questions. The appendix contains all proofs.

\section{Allocation stability}\label{sec:stability}

\subsection{Balancing counts}\label{sec:deterministic-counts}
We first look for a deterministic approximation to the pull counts. Recall that $c_T = \rho \log T$. For $n>0$, define the population index of arm $a$ as $\mu_a+B_a(n)$, where the population bonus is
\begin{align*}
 B_a(n)=\bigg(\sigma_a\sqrt{\frac{c_T}{n}}\bigg) \vee\frac{c_T}{n}.
\end{align*}
For $0<n<T$, write $H_T(n)=B_2(T-n)-\Delta-B_1(n)$. At allocation $(n,T-n)$, this is the population index of arm 2 minus that of arm 1.

\begin{proposition}\label{prop:deterministic-root}
For every $T$, the function $H_T$ is continuous and strictly increasing on $(0,T)$, with $\lim_{n\downarrow0}H_T(n)=-\infty$ and $\lim_{n\uparrow T}H_T(n)=+\infty$. It therefore has a unique zero $n_{1,T}^{\star}\in(0,T)$. Let $n_{2,T}^{\star}=T-n_{1,T}^{\star}$. Both $n_{1,T}^{\star}$ and $n_{2,T}^{\star}$ tend to infinity as $T\to\infty$.
\end{proposition}

The proof of \cref{prop:deterministic-root} is in Appendix~\ref{app:critical-expansion}. We call $n_{1,T}^{\star}$ and $n_{2,T}^{\star}$ the \emph{balancing counts}, which satisfy
\begin{align}
 B_2(n_{2,T}^{\star})=\Delta+B_1(n_{1,T}^{\star}),\qquad n_{1,T}^{\star}+n_{2,T}^{\star}=T.\label{eq:deterministic-root}
\end{align}
If we replace $B_a(n)$ by the population bonus for canonical UCB, \eqref{eq:deterministic-root} reduces to the two-arm balancing equation in \citet[Section~3.2]{han2024ucb}. These solutions give natural deterministic reference counts for the corresponding policies. Intuitively, pulling the arm with the larger population index lowers that index, making the other arm relatively more favorable. We would therefore expect the two population indices to be close after many pulls.

\subsection{The stability boundary}\label{sec:local-separation}
To state the result, define $\Delta_T^{\mathrm{crit}}=\sigma_2\sqrt{c_T/T}$ and write
\begin{align}
 r_T=\frac{\Delta}{\Delta_T^{\mathrm{crit}}},\qquad \beta_T=\frac{\sigma_1}{\sigma_2},\qquad \gamma_T=\frac{\sqrt{c_T/T}}{\sigma_2}. \label{eq:normalized-parameters}
\end{align}
Set
\begin{align*}
 \mathfrak d_T=|1-r_T|\vee\beta_T^{2/3}\vee\gamma_T^{1/2}.
\end{align*}

The following proposition describes the allocation fractions and the separation of the population indices around the balancing count when $\mathfrak d_T\to0$.

\begin{proposition}\label{prop:critical-behavior}
Suppose $\mathfrak d_T\to0$. Then
\begin{align}
 \frac{n_{1,T}^{\star}}T\to0,\qquad \frac{n_{1,T}}T\pto0.\label{eq:critical-localization-main}
\end{align}
For every fixed $\eps\in(0,1/2)$,
\begin{align}
 H_T((1+\eps)n_{1,T}^{\star}) &\asymp_\eps\Delta_T^{\mathrm{crit}}\mathfrak d_T,\label{eq:upper-nearby-count} \\
 -H_T((1-\eps)n_{1,T}^{\star}) &\asymp_\eps\Delta_T^{\mathrm{crit}}\mathfrak d_T.\label{eq:lower-nearby-count}
\end{align}
\end{proposition}

We prove \cref{prop:critical-behavior} in Appendix~\ref{app:critical-behavior}.

When $\mathfrak d_T\to0$, \cref{prop:critical-behavior} gives two useful conclusions. First, arm 2 receives nearly all observations, since $n_{1,T}/T\pto0$, so its sample-mean fluctuations have scale $\sigma_2/\sqrt T$. We note that although arm 1 receives fewer observations, its sample-mean fluctuations at deterministic counts near the balancing allocation are of smaller order than $\sigma_2/\sqrt T$, as are both arms' bonus-estimation errors; see Appendix~\ref{app:local-noise}. Thus the arm-2 sample mean is the leading source of noise in this local comparison.

Second, moving the arm-1 count a fixed fraction away from $n_{1,T}^{\star}$ gives a population-index separation of order $\Delta_T^{\mathrm{crit}}\mathfrak d_T$. To see whether estimation errors can reverse the index ordering, we therefore compare this separation with the arm-2 sample-mean fluctuation scale:
\begin{align*}
 \frac{\Delta_T^{\mathrm{crit}}\mathfrak d_T}{\sigma_2/\sqrt T} =\sqrt{c_T}\,\mathfrak d_T.
\end{align*}
The following theorem shows that arm 1 is stable if and only if this ratio diverges.

\begin{theorem}\label{thm:sharp}
Under the model and policy of \cref{sec:model-policy}:
\begin{enumerate}[label=(\roman*),leftmargin=2.2em]
\item Arm 1 is stable if and only if
\begin{align}
 \sqrt{c_T}\,\mathfrak d_T\to\infty.\label{eq:sharp-condition}
\end{align}
When this condition holds,
\begin{align}
 \frac{n_{1,T}}{n_{1,T}^{\star}}\pto1.\label{eq:optimal-root-stability}
\end{align}
If \eqref{eq:sharp-condition} fails, no deterministic sequence $m_T\to\infty$ satisfies $n_{1,T}/m_T\pto1$.
\item Arm 2 is stable for every sequence of reward laws, with
\begin{align*}
 \frac{n_{2,T}}{n_{2,T}^{\star}}\pto1.
\end{align*}
\end{enumerate}
\end{theorem}

We prove \cref{thm:sharp} in Appendix~\ref{app:sharp-proof}. Unlike canonical UCB, for which all arm counts are stable \citep[Theorem~3.1]{khamaru2024inference}, the variance-aware rule has a stability boundary that depends on the gap and the two reward variances.

We call a subsequence \emph{critical} if $\sqrt{c_T}\mathfrak d_T$ stays bounded along it, or equivalently if the gap and variances satisfy \eqref{eq:overview-critical-window} along that subsequence. Such a subsequence exists exactly when \eqref{eq:sharp-condition} fails.

\begin{remark}
There are two things to keep in mind here. First, arm 1 can still be stable when $\mathfrak d_T\to0$, as long as $\sqrt{c_T}\mathfrak d_T\to\infty$. Second, stability of arm 2 does not imply stability of arm 1, even though the two counts sum to $T$. The deviations from their balancing counts have the same magnitude, but stability concerns relative error. Along a critical subsequence, $n_{2,T}^{\star}\sim T$ while $n_{1,T}^{\star}=o(T)$, so an error that is small relative to $n_{2,T}^{\star}$ need not be small relative to $n_{1,T}^{\star}$.
\end{remark}

\begin{remark}\label{rem:bernoulli-stability}
We note that both counts are stable when the two arms have Bernoulli rewards, even if the success probabilities change with $T$. To see this, recall that in the Bernoulli case, we have $\sigma_a^2=\mu_a(1-\mu_a)$, so
\begin{align*}
 |\beta_T^2-1| =\frac{\Delta\,|1-\mu_1-\mu_2|}{\sigma_2^2} \leq\frac{\Delta}{\sigma_2^2} =r_T\gamma_T.
\end{align*}
By the definition of $\mathfrak d_T$, we then have $1 \leq\beta_T^2+|\beta_T^2-1|\leq\mathfrak d_T^3+(1+\mathfrak d_T)\mathfrak d_T^2=\mathfrak d_T^2+2\mathfrak d_T^3$. This gives $\mathfrak d_T>1/2$, so $\sqrt{c_T}\mathfrak d_T\to\infty$ and \cref{thm:sharp} applies.
\end{remark}

\subsection{The critical count limit}\label{sec:critical-count}
We next characterize the random allocation in the critical window. Define
\begin{align}
 \xi_T=\sqrt{c_T}(1-r_T),\qquad b_T=c_T^{3/4}\beta_T,\qquad g_T=c_T\gamma_T=\frac{c_T^{3/2}}{\sigma_2\sqrt T}.\label{eq:critical-coordinates}
\end{align}
These three coordinates are bounded along every critical subsequence. Write $S_{a,n}=\sum_{i=1}^n(X_{a,i}-\mu_a)$ for the centered partial sums, and set
\begin{align*}
 \chi_T=\frac{n_{1,T}\sqrt{c_T}}T,\qquad Y_T=\frac{S_{2,T}}{\sigma_2\sqrt T}.
\end{align*}
Here $Y_T$ is defined using the first $T$ rewards in the potential arm-2 sequence, including those not observed by the policy.

For $\xi\in\R$ and $b,g\geq0$ with $b+g>0$, let $\bar\chi(y)$ be the unique positive solution to
\begin{align*}
 \xi+y+\frac{\bar\chi(y)}2 = \frac b{\sqrt{\bar\chi(y)}}\vee\frac g{\bar\chi(y)}.
\end{align*}
Indeed, for every $y\in\R$, the function
\begin{align*}
    x\longmapsto \xi+y+\frac{x}{2} -\left(\frac b{\sqrt x}\vee\frac g{x}\right), \qquad x>0,
\end{align*}
is continuous and strictly increasing, with limits $-\infty$ as $x\downarrow0$ and $+\infty$ as $x\uparrow\infty$. When $b=g=0$, define instead $\bar\chi(y)=2(-\xi-y)_+$, where $z_+=\max\{z,0\}$.

\begin{proposition}\label{prop:critical-count}
Suppose that, along a critical subsequence,
\begin{align}
 (\xi_T,b_T,g_T)\longrightarrow(\xi,b,g).\label{eq:critical-coordinate-limits}
\end{align}
With $\bar\chi$ defined above for these limits,
\begin{align}
 \chi_T-\bar\chi(Y_T)\pto0,\label{eq:critical-count-approximation-main}
\end{align}
and, for $G_2\sim N(0,1)$,
\begin{align}
 (\chi_T,Y_T)\dto\bigl(\bar\chi(G_2),G_2\bigr).\label{eq:critical-count-limit}
\end{align}
The variable $\bar\chi(G_2)$ is nonconstant. It is strictly positive almost surely if $b+g>0$, whereas $\Prob\{\bar\chi(G_2)=0\}=\Phi(\xi)$ if $b=g=0$.
\end{proposition}

The proof is given in Appendix~\ref{app:critical-count}. To see the main idea, let $s_T=\sigma_2/\sqrt T$. In the critical window, we may first show that the arm-1 mean error and both bonus-estimation errors at the final counts are $\op(s_T)$ (see Appendix~\ref{app:local-terminal-balance}). Since $n_{2,T}/T\pto1$, one may replace the random count $n_{2,T}$ by $T$ in the arm-2 sample mean, giving $\bar X_{2,n_{2,T}}-\mu_2= \bar X_{2,T}-\mu_2 +\op(s_T) = s_TY_T+\op(s_T)$. Thus
\begin{align}\label{eq:two-indices}
 \UCB_1(n_{1,T}) = \mu_1+B_1(n_{1,T})+\op(s_T),\qquad  \UCB_2(n_{2,T}) = \mu_2+B_2(n_{2,T})+s_TY_T+\op(s_T).
\end{align}

Recall that $\Delta_T^{\mathrm{crit}}=\sigma_2\sqrt{c_T/T}$. Since $n_{2,T}/T=1-\chi_T/\sqrt{c_T}$, we may use the Taylor expansion $(1-z)^{-1/2}=1+z/2+O(z^2)$ to approximate the arm-2 bonus as
\begin{align*}
 B_2(n_{2,T})\approx\Delta_T^{\mathrm{crit}}\left(1-\frac{\chi_T}{\sqrt{c_T}}\right)^{-1/2}\approx\Delta_T^{\mathrm{crit}}+s_T\frac{\chi_T}{2}.
\end{align*}
Balancing the two indices in \eqref{eq:two-indices}, using $\Delta_T^{\mathrm{crit}}-\Delta=s_T\xi_T$, and dividing by $s_T$ then leads to
\begin{align*}
 \xi_T+Y_T+\frac{\chi_T}{2} =\left(\frac{b_T}{\sqrt{\chi_T}}\vee\frac{g_T}{\chi_T}\right)+\op(1).
\end{align*}
Solving this equation without the remainder and letting $(\xi_T,b_T,g_T)\to(\xi,b,g)$ gives \eqref{eq:critical-count-approximation-main}.

The function $\bar\chi$ is nonincreasing in $y$, and strictly decreasing when $b+g>0$. This is consistent with intuition: a positive arm-2 mean fluctuation raises the arm-2 index and tends to reduce the number of pulls of arm 1, while a negative fluctuation has the opposite effect. We note that this dependence between the allocation and the arm-2 mean error will also play a role in the Wald limits for coefficients that vary with $T$ studied in \cref{sec:uniform-targets} below.

\section{Wald inference for linear targets}\label{sec:inference}
In this section, we study Wald inference for linear combinations of the arm means. For a deterministic coefficient vector $u=(u_1,u_2)\neq0$, define
\begin{align*}
 \theta_T(u) =u_1\mu_1+u_2\mu_2,\qquad \hat\theta_T(u)=u_1\bar X_{1,n_{1,T}}+u_2\bar X_{2,n_{2,T}}.
\end{align*}
The choices $u=(1,0)$ and $(0,1)$ give the two arm means, while $u=(1,-1)$ gives their difference. The estimated standard errors are
\begin{align*}
 \hat s_{a,T}=\frac{\hat\sigma_{a,n_{a,T}}}{\sqrt{n_{a,T}}}, \quad a=1,2, \qquad \hat s_T(u)^2=\sum_{a=1}^2u_a^2\hat s_{a,T}^{\,2}.
\end{align*}
We study the ordinary Wald statistic
\begin{align}
 W_T(u)=\frac{\hat\theta_T(u)-\theta_T(u)}{\hat s_T(u)}\label{eq:linear-wald}
\end{align}
and its nominal $1-\alpha$ confidence interval $C_{T,\alpha}(u) =[\hat\theta_T(u)\pm z_{1-\alpha/2}\hat s_T(u)]$, where $z_p$ is the standard normal $p$th quantile. The statistic or standard-error ratio is set to zero whenever its denominator vanishes.

Throughout this section, we assume
\begin{align}
 n_{a,T}\sigma_a^2\pto\infty,\qquad a=1,2.\label{eq:joint-effective-sizes}
\end{align}
These conditions ensure relative consistency of the empirical variances, even when the reward variances tend to zero (see \cref{lem:random-index-variance}).

To see how stability gives a Gaussian limit, suppose that $n_{a,T}/n_{a,T}^{\star}\pto1$ for both arms. For any deterministic sequence $u=u_T\neq0$, write $s_T^{\star}(u)^2=\sum_{a=1}^2u_a^2\sigma_a^2/n_{a,T}^{\star}$. On the event that both empirical variances are positive, $W_T(u)$ has the decomposition
\begin{align}
 W_T(u) = \Biggl[\underbrace{\frac{1}{s_T^{\star}(u)}\sum_{a=1}^2\frac{u_aS_{a,n_{a,T}}}{n_{a,T}^{\star}}}_{\substack{\dto N(0,1)\\\text{martingale CLT}}} + \underbrace{\frac{1}{s_T^{\star}(u)}\sum_{a=1}^2\frac{u_aS_{a,n_{a,T}}}{n_{a,T}^{\star}}\left(\frac{n_{a,T}^{\star}}{n_{a,T}}-1\right) }_{\substack{\pto0\\\text{stability}}}\Biggr] \underbrace{\frac{s_T^{\star}(u)}{\hat s_T(u)}}_{\substack{\pto1\\\text{stability and}\\\text{variance consistency}}}. \label{eq:wald-stability-decomposition}
\end{align}
However, the decomposition here only shows that stability is sufficient for Gaussian inference, as in the canonical UCB results \citep{khamaru2024inference,han2024ucb}, while leaving open whether it is necessary.

\subsection{Fixed linear targets}\label{sec:fixed-targets}
In this subsection, we consider a fixed coefficient vector $u$, meaning that $u$ does not vary with the horizon $T$. Under the variance-aware policy studied here, we find that the usual Wald procedure is valid for every fixed nonzero $u$. Surprisingly, this remains true even when the optimal-arm count is unstable.

\begin{theorem}\label{thm:fixed-target}
Under \eqref{eq:joint-effective-sizes}, for every fixed nonzero $u\in\R^2$,
\begin{align*}
 W_T(u)\dto N(0,1).
\end{align*}
\end{theorem}

For each fixed $u\neq0$ and $\alpha\in(0,1)$, it follows that $\Prob\{\theta_T(u)\in C_{T,\alpha}(u)\}\longrightarrow1-\alpha$. This gives valid intervals for either arm mean and their difference without count stability. The proof uses a different count approximation, which we outline in the rest of this subsection.

The suboptimal-arm count is always stable by \cref{thm:sharp}. We therefore start with the optimal-arm mean and compare $n_{1,T}$ with a random count independent of the arm-1 rewards. Define an auxiliary policy by replacing the empirical arm-1 index with its population version, $\UCB_1^\circ(n)=\mu_1+B_1(n)$, while keeping the empirical arm-2 index. Both policies use the same arm-2 reward sequence, horizon, initialization, and tie-breaking rule. Let $n^{\circ}_{1,T}$ be the auxiliary policy's final arm-1 count, and set $\cF_{2,T}=\sigma(X_{2,1},\ldots,X_{2,T})$. The auxiliary policy is a proof device; the Wald statistic is computed from the original policy's observations.

\begin{proposition}\label{prop:cross-arm}
The count $n^{\circ}_{1,T}$ is $\cF_{2,T}$-measurable and independent of the arm-1 reward sequence. Moreover,
\begin{align}
 \frac{n_{1,T}}{n^{\circ}_{1,T}}\pto1.\label{eq:auxiliary-count-main}
\end{align}
\end{proposition}

The proof of \cref{prop:cross-arm} is in Appendix~\ref{app:cross-arm}. The argument compares the two policies at matching arm-2 counts. Their arm-2 indices are then identical because they use the same reward sequence and the fixed exploration factor $c_T$, even if the matching counts are reached at different rounds. Monotonicity and scaling of $B_1$, together with uniform control of the empirical index error, yield, with probability tending to one,
\begin{align}
 (1-\eta_T)^2(n^{\circ}_{1,T}-1)\leq n_{1,T}\leq(1+\eta_T)^2n^{\circ}_{1,T}+1,\label{eq:cross-arm-bound}
\end{align}
for some deterministic sequence $\eta_T\to0$.

To use this count for inference, write $R_t=X_{A_t,n_{A_t,t-1}+1}$ for the reward observed at round $t$ and let $\cF_{T,t}=\sigma(A_s,R_s:1\leq s\leq t)$, with $\cF_{T,0}$ trivial. We enlarge this filtration to $\cG_{T,t}=\cF_{T,t}\vee\cF_{2,T}$. The key point is that $n^{\circ}_{1,T}$ is now measurable at time zero, while each new arm-1 reward still has conditional mean $\mu_1$ and variance $\sigma_1^2$.

Consequently, $Z^{\circ}_{1,T}:=S_{1,n_{1,T}}/(\sigma_1\sqrt{n^{\circ}_{1,T}})$ is the terminal value of a martingale whose predictable quadratic variation is $n_{1,T}/n^{\circ}_{1,T}\pto1$. Since $n_{1,T}\sigma_1^2\pto\infty$, we have $n^{\circ}_{1,T}\sigma_1^2\pto\infty$, so bounded rewards give the conditional Lindeberg condition. The martingale CLT, together with variance consistency, gives
\begin{align*}
 Z_{1,T} := \frac{\bar X_{1,n_{1,T}} - \mu_1}{\hat s_{1,T}} = Z^{\circ}_{1,T}+\op(1) \dto N(0,1).
\end{align*}
See \cref{lem:cross-arm-gaussian} for the details. This handles targets with $u_2=0$, since $W_T(u)=\operatorname{sgn}(u_1)Z_{1,T}$.

Fix $u$ with $u_2\neq0$ and consider a critical subsequence. Arm 2 receives nearly all observations and remains stable, so $Z_{2,T} := (\bar X_{2,n_{2,T}}-\mu_2)/\hat s_{2,T}\dto N(0,1)$.  Also, we may show that $\hat s_{1,T}=\op(\sigma_2/\sqrt T)$ (see \eqref{eq:critical-terminal-errors} in the Appendix). Since $n_{2,T}/T\pto1$, variance consistency gives $\hat s_{2,T}/(\sigma_2/\sqrt T)\pto1$, and hence
\begin{align}\label{eq:variance-ratio}
    q_T:=\hat s_{1,T}/\hat s_{2,T}\pto0.
\end{align}
On the event of positive empirical variances, we therefore have
\begin{align}
 W_T(u) =\frac{u_1q_TZ_{1,T}+u_2Z_{2,T}}{\sqrt{u_1^2q_T^2+u_2^2}}=\operatorname{sgn}(u_2)Z_{2,T}+\op(1).  \label{eq:fixed-target-reduction}
\end{align}
Here $u$ is fixed with $u_2\neq 0$, and $q_TZ_{1,T}\pto0$ as $Z_{1,T}$ is bounded in probability.

On subsequences where $\sqrt{c_T}\mathfrak d_T\to\infty$, the earlier stability argument applies. Together with the critical case, this gives convergence along the full sequence.

\subsection{Uniformity over target coefficients}\label{sec:uniform-targets}

We now allow the coefficient vector to vary with the horizon and write it as $u_T$. This setting arises naturally in policy evaluation \citep{hadad2021confidence,zhan2021off}. A randomized policy that selects arm 1 with probability $p\in[0,1]$ and arm 2 with probability $1-p$ has value $V_T(p)=p\mu_1+(1-p)\mu_2=\theta_T(p,1-p)$. Thus, allowing the coefficients to vary with $T$ covers sequences of prespecified policies with deployment shares $p=p_T$. In particular, one may have $p_T\to1$, so that the coefficient on arm 2 tends to zero. In this case, the reduction in \eqref{eq:fixed-target-reduction} need not hold, as it relies on $u_2$ being fixed and nonzero.

To this end, let us fix a sequence of reward laws and define
\begin{align}
 D_T :=\sup_{u\in\R^2\setminus\{0\}}\sup_{z\in\R} \left|\Prob\{W_T(u)\leq z\}-\Phi(z)\right|.\label{eq:uniform-error}
\end{align}
Here the supremum is over deterministic coefficient vectors. Equivalently, $D_T\to0$ means that $W_T(u_T)\dto N(0,1)$ for every deterministic sequence $u_T\neq 0$.

\begin{theorem}\label{thm:uniform-wald}
Under \eqref{eq:joint-effective-sizes},
\begin{align*}
 D_T\longrightarrow0 \quad\Longleftrightarrow\quad \text{the optimal-arm count is stable}.
\end{align*}
\end{theorem}

\begin{remark}
For the canonical UCB algorithms studied by \citet{khamaru2024inference,han2024ucb}, both arm counts are stable. The same uniform Gaussian approximation therefore holds whenever the corresponding effective sample sizes diverge.
\end{remark}

The proof of \cref{thm:uniform-wald} is given in Appendix~\ref{app:uniform-wald}. By \cref{thm:sharp}, the optimal-arm count is stable if and only if $\sqrt{c_T}\mathfrak d_T\to\infty$. Sufficiency then follows from the martingale argument in \eqref{eq:wald-stability-decomposition}, which also applies to every deterministic sequence of coefficient vectors $u_T\neq 0$.

To see why stability is necessary, fix any sequence of reward laws for which the optimal-arm count is unstable. To rule out $D_T\to0$, it is enough to find one deterministic coefficient sequence whose Wald statistic has a nonstandard limit along a subsequence. By \cref{thm:sharp} and the boundedness of the critical coordinates, we may choose a critical subsequence satisfying \eqref{eq:critical-coordinate-limits}.

Along this subsequence, consider $u_T=(1,u_{2,T})$ with $u_{2,T}\geq0$ deterministic. On the event of positive empirical variances, we can write
\begin{align}
 W_T(1,u_{2,T}) = \frac{Z_{1,T}+w_TZ_{2,T}}{\sqrt{1+w_T^2}}, \qquad w_T = u_{2,T}\frac{\hat s_{2,T}}{\hat s_{1,T}}. \label{eq:moving-weight-wald}
\end{align}
For fixed $u_{2,T}=u_2>0$, \eqref{eq:variance-ratio} gives $w_T\pto\infty$ and hence $W_T(1,u_2)=Z_{2,T}+\op(1)$. We therefore need to choose $u_{2,T}$ so that $w_T$ has a finite, nonconstant limit.

Recall that $\chi_T=\sqrt{c_T}\,n_{1,T}/T$. Relative variance consistency (see \cref{lem:random-index-variance}) and $n_{2,T}/T\pto1$ give, on the same event,
\begin{align*}
 w_T=\frac{u_{2,T}}{(\sigma_1/\sigma_2)c_T^{1/4}} \sqrt{\chi_T}\{1+\op(1)\}.
\end{align*}
Take $u_{2,T}=\tau(\sigma_1/\sigma_2)c_T^{1/4}$ for a fixed $\tau>0$. Recall also that $Y_T=S_{2,T}/(\sigma_2\sqrt T)$. \cref{prop:critical-count} and the preceding display give
\begin{align*}
 w_T=\tau\sqrt{\chi_T}+\op(1) =\tau\sqrt{\bar\chi(Y_T)}+\op(1).
\end{align*}
Let $G_2\sim N(0,1)$ denote the limit of $Y_T$. Since $Z_{2,T}=Y_T+\op(1)$, $Z_{2,T}$ and $w_T$ converge jointly to $G_2$ and $\tau\sqrt{\bar\chi(G_2)}$. This convergence also holds jointly with $Z_{1,T}\dto G_1$, where $G_1\sim N(0,1)$ is independent of $G_2$; see \cref{lem:critical-joint-limit}.

\begin{corollary}\label{cor:critical-coefficients}
Assume \eqref{eq:joint-effective-sizes} and fix a critical subsequence satisfying \eqref{eq:critical-coordinate-limits}. All limits in this corollary are taken along this subsequence. Let $\bar\chi$ be the function in \cref{prop:critical-count}. For a deterministic sequence $u_{2,T}\geq0$ and all sufficiently large $T$, let
\begin{align}
 \tau_T=\frac{u_{2,T}}{(\sigma_1/\sigma_2)c_T^{1/4}}.\label{eq:coefficient-scale}
\end{align}
If $\tau_T\to\tau\in[0,\infty)$, then
\begin{align}
 W_T(1,u_{2,T})\dto W_\tau :=\frac{G_1+\tau\sqrt{\bar\chi(G_2)}G_2} {\sqrt{1+\tau^2\bar\chi(G_2)}}, \qquad G_1,G_2\overset{\rm iid}{\sim}N(0,1).\label{eq:finite-coefficient-limit}
\end{align}
For $\tau=0$ this limit is standard normal. For every $\tau\in(0,\infty)$, it satisfies $\E W_\tau<0$ and is not standard normal. If $b+g>0$ and $\tau_T\to\infty$, then
\begin{align}
 W_T(1,u_{2,T})\dto N(0,1).\label{eq:large-coefficient-limit}
\end{align}
\end{corollary}

The proof of \cref{cor:critical-coefficients} is given in Appendix~\ref{app:critical-coefficients}. Taking $u_{2,T}=(\sigma_1/\sigma_2)c_T^{1/4}$ gives $\tau_T=1$ for all sufficiently large $T$. \Cref{cor:critical-coefficients} therefore gives a nonstandard limit for $W_T(1,u_{2,T})$ along the chosen subsequence. Hence $D_T$ cannot tend to zero, proving necessity in \cref{thm:uniform-wald}.

\begin{remark}
For $0<\tau<\infty$, the negative mean of $W_\tau$ comes from the dependence between $G_2$ and its weight $\tau\sqrt{\bar\chi(G_2)}/\sqrt{1+\tau^2\bar\chi(G_2)}$. This weight is nonconstant and nonincreasing in $G_2$, so it is negatively correlated with $G_2$. Since $\E G_2=0$ and the $G_1$ term has mean zero by independence, we obtain $\E W_\tau<0$.
\end{remark}

\begin{remark}
The condition $b+g>0$ in \eqref{eq:large-coefficient-limit} ensures that $\bar\chi(G_2)>0$ almost surely. Together with $\tau_T\to\infty$, this gives $w_T\pto\infty$ and hence $W_T(1,u_{2,T})=Z_{2,T}+\op(1)$. When $b=g=0$, however, $\bar\chi(G_2)$ has an atom at zero. The count limit then gives too little information about small values of $\chi_T$ to conclude that $w_T=\tau_T\sqrt{\chi_T}\{1+\op(1)\}$ diverges.
\end{remark}

\begin{example}\label{ex:fixed-moving}
We compare the fixed coefficients $(1,0)$ with a sequence $u_T\to(1,0)$. Fix $\bar\sigma\in(0,1/2]$ and set $\Delta=\bar\sigma(\sqrt{c_T}-1)/\sqrt T$. Consider
\begin{align}
 X_{2,i}=\tfrac12+\bar\sigma\varepsilon_{2,i},\qquad X_{1,i}=\tfrac12+\Delta+\frac{\bar\sigma}{4c_T^{3/4}}\varepsilon_{1,i}, \label{eq:example-rewards}
\end{align}
where the $\varepsilon_{a,i}$ are independent and take the values $-1$ and $1$ with equal probability. For all sufficiently large $T$, these rewards lie in $[0,1]$, and
\begin{align*}
 \sigma_2=\bar\sigma,\qquad \sigma_1=\frac{\bar\sigma}{4c_T^{3/4}},\qquad (\xi_T,b_T,g_T) =\left(1,\tfrac14,\frac{c_T^{3/2}}{\bar\sigma\sqrt T}\right) \longrightarrow(1,\tfrac14,0).
\end{align*}
The optimal-arm count is unstable by \cref{thm:sharp}.  We may also check that both effective sample sizes diverge. Hence $W_T(1,0)\dto N(0,1)$ by \cref{thm:fixed-target}. For $u_T=(1,(2\sqrt{c_T})^{-1})\longrightarrow(1,0)$, however, we have $\tau_T=2$. \Cref{cor:critical-coefficients} gives
\begin{align}
 W_T(u_T)\dto W_2 :=\frac{G_1+2\sqrt{\bar\chi(G_2)}\,G_2}{\sqrt{1+4\bar\chi(G_2)}}, \label{eq:W2-def}
\end{align}
where $G_1,G_2$ are independent standard normal variables and $\bar\chi(y)>0$ solves $1+y+\bar\chi(y)/2=1/(4\sqrt{\bar\chi(y)})$. The limit is not standard normal, even though $u_T\to(1,0)$.
\end{example}

\begin{remark}
Although we study a variance-aware UCB policy, we expect the proof approach to extend to other index policies. The sufficiency argument already applies whenever both arm counts are stable and \eqref{eq:joint-effective-sizes} holds. Extending the necessity argument would require suitable analogues of \cref{prop:cross-arm,prop:critical-count}. These would need to describe how unstable counts depend on reward fluctuations and whether this dependence prevents a uniform Gaussian approximation.
\end{remark}

\section{Comparison with canonical UCB}\label{sec:comparison}

In this section, we compare regret and Wald interval lengths under variance-aware and canonical UCB, focusing on the optimal-arm mean and the mean difference.

For canonical UCB \citep{auer2002finite}, we use the fixed-horizon index
\begin{align}
 I_a^U(n)=\bar X_{a,n}+\sqrt{\frac{c_T}{2n}}, \qquad c_T=\rho\log T,\label{eq:comparison-ucb-index}
\end{align}
with the same horizon, initialization, and exploration constant as the variance-aware rule. Superscripts $U$ and $V$ denote canonical and variance-aware UCB, respectively. The two policies use the same reward sequences. Let $A_t^\pi$ be the arm chosen at round $t$ under policy $\pi\in\{U,V\}$. The random pseudo-regret \citep[Chapter~4]{lattimore2020bandit} is
\begin{align}
 \mathrm{Reg}_T^\pi =\sum_{t=1}^T(\mu_1-\mu_{A_t^\pi}) =\Delta n_{2,T}^\pi.\label{eq:comparison-regret}
\end{align}
All ratio limits below are in probability. For fixed-gap regret limits under UCB and Thompson sampling, see \citet{fan2022typical}.

Fix $\alpha\in(0,1)$ and let $\hat\sigma_{a,T}^\pi$ be the empirical standard deviation at the final arm-$a$ count under policy $\pi$. The nominal $1-\alpha$ interval lengths for $\mu_1$ and $\Delta$ are
\begin{align}
 L_{1,T}^{\pi} = 2z_{1-\alpha/2}\frac{\hat\sigma_{1,T}^{\pi}}{\sqrt{n_{1,T}^{\pi}}}, \qquad L_{\Delta,T}^{\pi} =2z_{1-\alpha/2} \sqrt{\frac{(\hat\sigma_{1,T}^{\pi})^2}{n_{1,T}^{\pi}}+\frac{(\hat\sigma_{2,T}^{\pi})^2}{n_{2,T}^{\pi}}}. \label{eq:comparison-lengths}
\end{align}
In the two settings below, these intervals have asymptotic coverage $1-\alpha$ under both policies. The proofs are in Appendix~\ref{app:interval-lengths}.

\subsection{Fixed reward laws}\label{sec:comparison-fixed}

Suppose the reward laws are fixed, with $\sigma_1,\sigma_2>0$ and $\Delta>0$. The arm-2 counts satisfy
\begin{align}
 n_{2,T}^{U} = \{1+\op(1)\}\frac{c_T}{2\Delta^2}, \qquad n_{2,T}^{V}=\{1+\op(1)\}c_T\left(\frac{\sigma_2^2}{\Delta^2}\vee\frac1\Delta\right),\label{eq:length-fixed-counts}
\end{align}
while $n_{1,T}^{\pi}/T\pto1$ under either policy. Both terms in the variance-aware bonus enter the count approximation. The standard-deviation term determines the first-order count when $\Delta<\sigma_2^2$, and the inverse-count term does so when $\Delta>\sigma_2^2$. The count approximations give
\begin{align}
 \frac{\mathrm{Reg}_T^V}{\mathrm{Reg}_T^U} \pto2(\sigma_2^2\vee\Delta).\label{eq:regret-fixed-ratio}
\end{align}
Together with variance consistency, they also give
\begin{align}
 \frac{L_{1,T}^{V}}{L_{1,T}^{U}}\pto1,\qquad \frac{L_{\Delta,T}^{V}}{L_{\Delta,T}^{U}} \pto\frac1{\sqrt{2(\sigma_2^2\vee\Delta)}}. \label{eq:length-fixed-ratios}
\end{align}
Since rewards lie in $[0,1]$, we have $\sigma_2^2\leq1/4$, so $2(\sigma_2^2\vee\Delta)<1$ if and only if $\Delta<1/2$. Thus variance-aware UCB has asymptotically lower regret, and a longer mean-difference interval, exactly when $\Delta<1/2$. The two policies agree to first order when $\Delta=1/2$, and both orderings reverse when $\Delta>1/2$. When $\Delta<\sigma_2^2$, the standard-deviation term determines the arm-2 count and the regret ratio is $2\sigma_2^2\leq1/2$, in line with the low-variance benefit in the regret bounds of \citet{audibert2009exploration}. When $\Delta>\sigma_2^2$, the ratio $2\Delta$ instead compares the inverse-count term $c_T/n$ with the canonical bonus $\sqrt{c_T/(2n)}$, so its value depends on the bonus normalizations in the two policies. Since both policies allocate nearly all observations to arm 1, their intervals for $\mu_1$ have the same first-order length.

\subsection{The region \texorpdfstring{$\mathfrak d_T\to0$}{dT tending to zero}}
\label{sec:comparison-near}

When the reward laws vary with $T$, suppose that
\begin{align}
 \mathfrak d_T\to0,\qquad \sigma_2\to \bar\sigma\in[0,1/2],\qquad n_{a,T}^{V}\sigma_a^2\pto\infty,\quad a=1,2. \label{eq:comparison-near-cancellation}
\end{align}
Here $\Delta/\Delta_T^{\mathrm{crit}}\to1$, $\sigma_1/\sigma_2\to0$, and $\gamma_T\to0$. These conditions allow the optimal-arm count to be either unstable or stable (with $\sqrt{c_T}\mathfrak d_T\to\infty$).

Note that the function $\lambda\mapsto\lambda^{-1/2}-(1-\lambda)^{-1/2}$ decreases continuously from $+\infty$ to $0$ on $(0,1/2]$. There is therefore a unique $\lambda^U_{\bar\sigma}\in(0,1/2]$ satisfying
\begin{align}
 \frac1{\sqrt{\lambda^U_{\bar\sigma}}}-\frac1{\sqrt{1-\lambda^U_{\bar\sigma}}} =\sqrt2\,\bar\sigma.\label{eq:length-critical-fraction}
\end{align}
Under canonical UCB, the fractions of pulls from arms 1 and 2 converge to $1-\lambda^U_{\bar\sigma}$ and $\lambda^U_{\bar\sigma}$, respectively. The balancing equation has the same form as the small- and moderate-gap limits in \citet[Theorem~1]{kalvit2021closer}, with constants corresponding to \eqref{eq:comparison-ucb-index}. Under variance-aware UCB, arm 2 instead receives nearly all observations: \cref{prop:critical-behavior} gives $n_{1,T}^{V}=\op(T)$ and $n_{2,T}^{V}=T-\op(T)$. The regret ratio is therefore
\begin{align}
 \frac{\mathrm{Reg}_T^V}{\mathrm{Reg}_T^U} =\frac{n_{2,T}^V/T}{n_{2,T}^U/T} \pto\frac1{\lambda^U_{\bar\sigma}}\geq2.\label{eq:regret-critical-ratio}
\end{align}
The inequality is strict when $\bar\sigma>0$; when $\bar\sigma=0$, $\lambda^U_{\bar\sigma}=1/2$.

For the mean difference, the arm-1 variance term is negligible under both policies. For variance-aware UCB, we can show (see \eqref{eq:length-near-arm1-negligible} in Appendix~\ref{app:comparison-near}) that
\begin{align*}
 \frac{\sigma_1^2}{n_{1,T}^{V}}=\op \left(\frac{\sigma_2^2}{T}\right).
\end{align*}
For canonical UCB, $n_{1,T}^{U}/T\pto1-\lambda^U_{\bar\sigma}>0$ (see \eqref{eq:length-critical-u-counts} in Appendix~\ref{app:comparison-near}), so $\sigma_1/\sigma_2\to0$ gives the same conclusion. The count limits and relative variance consistency then give
\begin{align}
 \frac{L_{1,T}^{V}}{L_{1,T}^{U}}\pto\infty,\qquad \frac{L_{\Delta,T}^{V}}{L_{\Delta,T}^{U}}\pto\sqrt{\lambda^U_{\bar\sigma}}<1. \label{eq:length-critical-ratios}
\end{align}
Therefore, the mean-difference interval is shorter under variance-aware UCB because arm 2 receives more observations. The interval for $\mu_1$ is longer by a diverging factor: arm 1 receives a vanishing fraction of the budget under variance-aware UCB, but a positive fraction under canonical UCB.

\begin{remark}
We can also compare with a fixed allocation based on the true variances. The Neyman allocation minimizes $\sigma_1^2/n+\sigma_2^2/(T-n)$ over real counts $0<n<T$, giving $(\sigma_1+\sigma_2)^2/T$. Adaptive Neyman allocation is studied by \citet{zhao2023adaptive} and, under a design-based variance criterion, by \citet{dai2023clip}. Under variance-aware UCB, the arm-1 variance term is negligible and $n_{2,T}^{V}/T\pto1$. Since $\sigma_1/\sigma_2\to0$, we obtain
\begin{align}
 \frac{L_{\Delta,T}^{V}} {2z_{1-\alpha/2}(\sigma_1+\sigma_2)/\sqrt T}\pto1. \label{eq:length-neyman-benchmark}
\end{align}
Thus the mean-difference interval has the same first-order length as under Neyman allocation, even though the arm-1 counts may differ on their own scale.
\end{remark}

Table~\ref{tab:interval-lengths} summarizes the two comparisons.
\begin{table}[!htb]
\centering
\caption{Probability limits of variance-aware/canonical UCB ratios.}
\label{tab:interval-lengths}
\begin{tabular}{@{}lccc@{}}
\toprule
 & Pseudo-regret & Optimal mean & Mean difference\\
 & $\mathrm{Reg}_T^V/\mathrm{Reg}_T^U$ & $L_{1,T}^{V}/L_{1,T}^{U}$ & $L_{\Delta,T}^{V}/L_{\Delta,T}^{U}$\\
\midrule
Fixed reward laws & $2(\sigma_2^2\vee\Delta)$ & $1$ & $\{2(\sigma_2^2\vee\Delta)\}^{-1/2}$\\
$\mathfrak d_T\to0$ under \eqref{eq:comparison-near-cancellation} & $1/\lambda^U_{\bar\sigma}\geq2$ & $+\infty$ & $\sqrt{\lambda^U_{\bar\sigma}}<1$\\
\bottomrule
\end{tabular}
\end{table}

In both settings, the arm-2 count determines the squared mean-difference interval length to first order:
\begin{align}
 (L_{\Delta,T}^\pi)^2 =\{1+\op(1)\}\, 4z_{1-\alpha/2}^2\frac{\sigma_2^2}{n_{2,T}^\pi}, \qquad\pi\in\{U,V\}.\label{eq:comparison-arm2-length}
\end{align}
Combining this with \eqref{eq:comparison-regret} gives
\begin{align}
 \frac{\mathrm{Reg}_T^V}{\mathrm{Reg}_T^U} \left(\frac{L_{\Delta,T}^V}{L_{\Delta,T}^U}\right)^2 \pto1.\label{eq:regret-length-reciprocity}
\end{align}
For these two policies in either setting, lower regret therefore comes with a longer mean-difference interval, consistent with the regret--estimation tradeoff studied by \citet{erraqabi2017trading,simchi2025design}. More pulls of arm 2 increase regret but shorten the mean-difference interval. The interval for $\mu_1$ depends instead on the arm-1 count.

\section{Numerical illustration}\label{sec:numerics}

In this section, we illustrate the inference results for fixed and changing targets under variance-aware and canonical UCB. The reward laws are those in \cref{ex:fixed-moving}, with $\bar\sigma=1/2$, $\rho=2$, and $T=10^8$. The fixed targets are $(1,0)$ and $(1,-1)$, and the changing target is $u_T=(1,(2\sqrt{c_T})^{-1})$.

We run 6,000 independent pairs of experiments, using the variance-aware rule and the canonical rule in \eqref{eq:comparison-ucb-index}. Within each pair, the two policies use the same reward sequence for each arm. Both policies first pull arm 1 and then arm 2, with ties resolved in favor of arm 1.

Figure~\ref{fig:wald-histograms} compares the distributions of $W_T(1,0)$ and $W_T(u_T)$ under the two policies. For the fixed target, the Wald statistic has a standard normal limit under both policies by \cref{thm:fixed-target,lem:stable-wald}. For the changing target, the limit is the nonstandard $W_2$ in \cref{ex:fixed-moving} under variance-aware UCB, but remains standard normal under canonical UCB.

\begin{figure}[!htb]
 \centering
 \begin{minipage}[t]{0.49\textwidth}
  \centering
  \includegraphics[width=\linewidth]{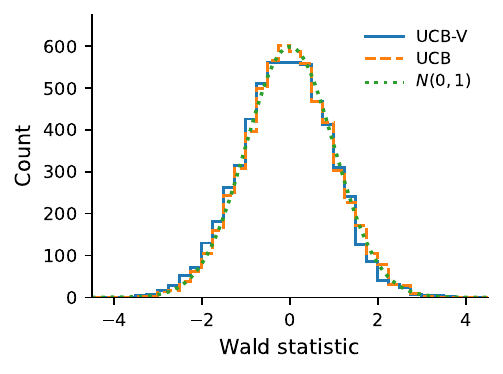}
  \par\small (a) Fixed target $(1,0)$
 \end{minipage}\hfill
 \begin{minipage}[t]{0.49\textwidth}
  \centering
  \includegraphics[width=\linewidth]{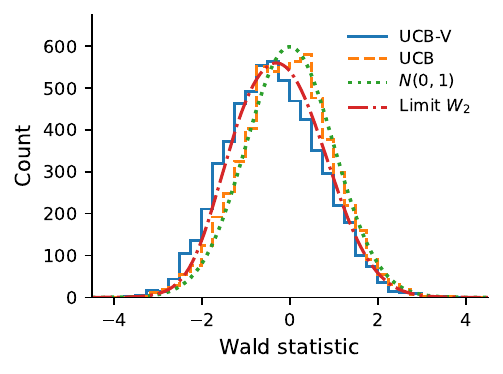}
  \par\small (b) Changing target $u_T$
 \end{minipage}
 \caption{Histograms of $W_T(1,0)$ and $W_T(u_T)$ under the two policies, based on 6,000 replications. All histograms use bins of width $0.25$ and the same axes. The reference curves are the standard normal density and, in panel (b), the density of $W_2$ defined in \eqref{eq:W2-def}.}
 \label{fig:wald-histograms}
\end{figure}

Table~\ref{tab:wald-coverage} reports the empirical coverage of the nominal 90\% and 95\% Wald intervals for all three targets. For the optimal mean, coverage is close to the nominal level under both policies. For the changing target, the 95\% coverage is 92.87\% under variance-aware UCB and 94.48\% under canonical UCB. Overall, these empirical coverage patterns are well aligned with our theoretical predictions.

\begin{table}[!htb]
 \centering
 \small
 \setlength{\tabcolsep}{10pt}
 \caption{Empirical coverage, in percent, of nominal 90\% and 95\% Wald intervals.}
 \label{tab:wald-coverage}
 \begin{tabular}{@{}llrr@{}}
  \toprule
  Policy & Target & 90\% & 95\%\\
  \midrule
  UCB-V & $(1,0)$ & 89.52 & 94.57 \\
  & $(1,-1)$ & 88.35 & 93.95 \\
  & $u_T$ & 85.97 & 92.87 \\
  \midrule
  UCB & $(1,0)$ & 89.83 & 94.70 \\
  & $(1,-1)$ & 89.02 & 94.45 \\
  & $u_T$ & 89.12 & 94.48 \\
  \bottomrule
 \end{tabular}
\end{table}

\section{Discussion}\label{sec:discussion}
We have studied the role of allocation stability in ordinary Wald inference for a two-armed variance-aware UCB policy. Under diverging effective sample sizes, the usual Wald statistic has a standard normal limit for every fixed nonzero coefficient vector, even when the optimal-arm count is unstable. Gaussian approximation uniformly over deterministic coefficient vectors holds if and only if this count is stable.

Our results are asymptotic, and the numerical experiments in \cref{sec:numerics} show coverage errors even for fixed targets with standard normal limits. It would be useful to quantify how these errors depend on the horizon, the gap, and the two reward variances. This would help determine when the usual Wald intervals have coverage close to the nominal level, particularly near the stability boundary.

Our analysis of coefficients that vary with the horizon is incomplete in the case $b=g=0$ and $\tau_T\to\infty$. The limiting rescaled count then has an atom at zero, so it does not determine the order of the arm-1 counts that are smaller than $T/\sqrt{c_T}$. These counts can still affect the relative contributions of the two mean errors to the Wald statistic. A more precise description of their size and dependence on reward errors is needed to identify the remaining limits in \cref{cor:critical-coefficients}.

It would also be interesting to study bonuses that depend on the current round and policies with more than two arms. In our auxiliary-policy comparison, the fixed exploration factor makes the arm-2 indices identical at matching counts. If the bonus varied with the round, we would need to compare both the counts and the rounds at which they are reached. With more than two arms, the standard error can involve several random counts, and we would need to approximate these counts jointly with the corresponding mean errors.

\bibliographystyle{unsrtnat}
\bibliography{references}

\clearpage
\appendix

\section{Proofs for Section~\ref{sec:stability}: Allocation stability}\label{app:stability}

In the appendices, we also use $N_T$ as shorthand for $n_{1,T}$.

\subsection{Proof of Proposition~\ref{prop:deterministic-root}}\label{app:critical-expansion}
\begin{proof}
Each $B_a$ is continuous and strictly decreasing on $(0,\infty)$, and $B_a(n)\geq c_T/n$. Hence, for $0<n<m<T$,
\begin{align*}
 H_T(m)-H_T(n) =B_2(T-m)-B_2(T-n)+B_1(n)-B_1(m)>0.
\end{align*}
Also, $\lim_{n\downarrow0}H_T(n)=-\infty$ and $\lim_{n\uparrow T}H_T(n)=+\infty$. Continuity and strict monotonicity give exactly one zero.

If $n_{1,T}^{\star}\leq K$ along a subsequence, then the balancing identity gives
\begin{align*}
 \frac{c_T}{K} \leq B_1(n_{1,T}^{\star}) =B_2(n_{2,T}^{\star})-\Delta \leq B_2(T-K)\longrightarrow0,
\end{align*}
contradicting $c_T/K\to\infty$. If $n_{2,T}^{\star}\leq K$ along a subsequence, bounded rewards instead give
\begin{align*}
 \frac{c_T}{K} \leq B_2(n_{2,T}^{\star}) =\Delta+B_1(n_{1,T}^{\star}) \leq1+B_1(T-K)=1+o(1),
\end{align*}
which is again impossible. Both balancing counts therefore diverge. This completes the proof.
\end{proof}

\subsection{Terminal index comparisons}\label{app:terminal-comparison}

For $\kappa>0$ and $n>0$, comparing the two bonus terms gives
\begin{align}
 \min\{\kappa^{-1/2},\kappa^{-1}\}B_a(n) \leq B_a(\kappa n) \leq\max\{\kappa^{-1/2},\kappa^{-1}\}B_a(n). \label{eq:bonus-scaling}
\end{align}
In particular, for integers $n\geq2$,
\begin{align}
 0\leq B_a(n-1)-B_a(n)\leq\frac{B_a(n)}{n-1}. \label{eq:bonus-one-step}
\end{align}
Let $\eta_T$ and $\cE_T^{\rm unif}$ be as in \eqref{eq:eta-ucbv} and \eqref{eq:uniform-ucbv-event}, respectively.

\begin{lemma}\label{lem:terminal-balance}
For all sufficiently large $T$, every reward path satisfies
\begin{align}
 \min_{a=1,2}n_{a,T}>c_T/2.\label{eq:terminal-count-lower}
\end{align}
Moreover,
\begin{align}
 \mathcal R_T:=\frac{B_2(n_{2,T})}{\Delta+B_1(n_{1,T})} \pto1.\label{eq:terminal-relative-balance}
\end{align}
\end{lemma}
\begin{proof}
Write $n_a=n_{a,T}$. Suppose $n_a\leq c_T/2$ for some arm $a$, and let $a'\neq a$ denote the other arm. Since $n_{a'}\geq T-c_T/2>1$ for all large $T$, its last pull occurs after initialization. Just before that pull, its count is $n_{a'}-1\geq T-c_T/2-1$, while arm $a$ has some count $1\leq j\leq n_a$. Since rewards lie in $[0,1]$ and empirical standard deviations are at most $1/2$, selection of arm $a'$ requires
\begin{align*}
 2\leq\frac{c_T}{n_a} \leq\UCB_a(j) \leq\UCB_{a'}(n_{a'}-1) \leq1+\left(\frac12\sqrt{\frac{c_T}{T-c_T/2-1}} \vee\frac{c_T}{T-c_T/2-1}\right)<2
\end{align*}
for all sufficiently large $T$. This proves \eqref{eq:terminal-count-lower}.

On $\cE_T^{\rm unif}$, the last pull of arm 1 and the last pull of arm 2 give, respectively,
\begin{align*}
 (1-\eta_T)B_2(n_2) &\leq\Delta+(1+\eta_T)B_1(n_1-1),\\
 \Delta+(1-\eta_T)B_1(n_1) &\leq(1+\eta_T)B_2(n_2-1).
\end{align*}
At each of these pulls, the other arm's count is at most its terminal count. The bounds therefore follow from monotonicity. By \eqref{eq:bonus-one-step} and \eqref{eq:terminal-count-lower},
\begin{align*}
 B_a(n_a-1)\leq(1+\delta_T)B_a(n_a), \qquad \delta_T:=\frac{1}{c_T/2-1}\to0.
\end{align*}
Since $\Delta>0$, the two selection inequalities imply
\begin{align*}
 \frac{1-\eta_T}{(1+\eta_T)(1+\delta_T)} \leq\mathcal R_T \leq\frac{(1+\eta_T)(1+\delta_T)}{1-\eta_T}.
\end{align*}
Now use $\Prob(\cE_T^{\rm unif})\to1$ to conclude.
\end{proof}

\subsubsection{Stability of arm 2}\label{app:arm2-stability}
\begin{lemma}\label{lem:arm2-stability}
For every sequence of reward laws, $n_{2,T}/n_{2,T}^{\star}\pto1$.
\end{lemma}
\begin{proof}
Write $n_a=n_{a,T}$ and $n_a^\star=n_{a,T}^{\star}$, so that $B_2(n_2^\star)=\Delta+B_1(n_1^\star)$. Fix $\eps\in(0,1)$. By \eqref{eq:bonus-scaling} and $n_1+n_2=n_1^\star+n_2^\star=T$,
\begin{align*}
 n_2\geq(1+\eps)n_2^\star &\quad\Longrightarrow\quad \mathcal R_T \leq\frac{B_2((1+\eps)n_2^\star)}{\Delta+B_1(n_1^\star)} \leq(1+\eps)^{-1/2},\\
 n_2\leq(1-\eps)n_2^\star &\quad\Longrightarrow\quad \mathcal R_T \geq\frac{B_2((1-\eps)n_2^\star)}{\Delta+B_1(n_1^\star)} \geq(1-\eps)^{-1/2}.
\end{align*}
Both events have probability tending to zero by \eqref{eq:terminal-relative-balance}. This completes the proof.
\end{proof}

\subsection{Proof of Proposition~\ref{prop:critical-behavior}}\label{app:critical-behavior}

\subsubsection{The balancing equation near its root}
With the parameters from \eqref{eq:normalized-parameters}, write
\begin{align*}
  \lambda_T^\star :=\frac{n_{1,T}^\star}{T}, \qquad \mathcal H_T(\lambda):=\frac{H_T(T\lambda)}{\Delta_T^{\mathrm{crit}}},\qquad h(\lambda):=(1-\lambda)^{-1/2}-1,\qquad \psi_T(\lambda):=\frac{\beta_T}{\sqrt\lambda}\vee\frac{\gamma_T}{\lambda}.
\end{align*}
In these units, $\psi_T(\lambda)$ is the arm-1 bonus at count $T\lambda$, while $h(\lambda)$ is the increase in the arm-2 standard-deviation bonus when its count falls from $T$ to $T(1-\lambda)$. For $0<\lambda<1$,
\begin{align}
  \mathcal H_T(\lambda) =\left(\frac{1}{\sqrt{1-\lambda}}\vee\frac{\gamma_T}{1-\lambda}\right) -r_T-\psi_T(\lambda).\label{eq:dimensionless-H}
\end{align}

The condition $\mathfrak d_T\to0$ is equivalent to
\begin{align}
  r_T\to1,\qquad\beta_T\to0,\qquad\gamma_T\to0.\label{eq:singular-regime}
\end{align}

\begin{lemma}\label{lem:local-equation}
Under \eqref{eq:singular-regime}, $\lambda_T^\star\to0$, and, for all sufficiently large $T$,
\begin{align}
  \frac1{\sqrt{1-\lambda_T^\star}}-r_T =\frac{\beta_T}{\sqrt{\lambda_T^\star}}\vee\frac{\gamma_T}{\lambda_T^\star}.\label{eq:local-fixed-point}
\end{align}
\end{lemma}

\begin{proof}
At the root, \eqref{eq:dimensionless-H} reads
\begin{align*}
  \left(\frac1{\sqrt{1-\lambda_T^\star}}\vee\frac{\gamma_T}{1-\lambda_T^\star}\right) =r_T+\psi_T(\lambda_T^\star).
\end{align*}
If $\lambda_T^\star\geq\delta>0$ along a subsequence, then
\begin{align*}
 \frac1{\sqrt{1-\delta}} \leq\frac1{\sqrt{1-\lambda_T^\star}} \leq r_T+\psi_T(\lambda_T^\star) \leq r_T+\frac{\beta_T}{\sqrt\delta}+\frac{\gamma_T}{\delta} \longrightarrow1,
\end{align*}
contradicting $(1-\delta)^{-1/2}>1$. Thus $\lambda_T^\star\to0$.

Since $\gamma_T\to0$ as well, we have $(1-\lambda_T^\star)^{-1/2}>\gamma_T/(1-\lambda_T^\star)$ for all large $T$. Substituting the standard-deviation branch for the arm-2 bonus gives \eqref{eq:local-fixed-point}.
\end{proof}

With this simplification, the balancing equation reads
\begin{align}
  \psi_T(\lambda_T^\star)=(1-r_T)+h(\lambda_T^\star).\label{eq:psi-root}
\end{align}

The powers in $\mathfrak d_T$ can be actually read from this equation. When $r_T=1$, we balance $h(\lambda_T^\star)\asymp\lambda_T^\star$ with $\psi_T(\lambda_T^\star)$. If the standard-deviation term determines $\psi_T$ at the root, then $(\lambda_T^\star)^{3/2}\asymp\beta_T$; if the inverse-count term determines it, then $(\lambda_T^\star)^2\asymp\gamma_T$. These give the scales $\beta_T^{2/3}$ and $\gamma_T^{1/2}$. For general $r_T$, the term $1-r_T$ also enters the equation. The next lemma relates $\mathfrak d_T$ to $\lambda_T^\star+\psi_T(\lambda_T^\star)$.

\begin{lemma}\label{lem:root-scale}
Under \eqref{eq:singular-regime},
\begin{align}
  \lambda_T^\star+\psi_T(\lambda_T^\star)\asymp\mathfrak d_T.\label{eq:root-scale}
\end{align}
In particular, $\lambda_T^\star\lesssim\mathfrak d_T$ and $B_1(n_{1,T}^\star)\lesssim\Delta_T^{\mathrm{crit}}\mathfrak d_T$.
\end{lemma}

\begin{proof}
For $0<\lambda\leq1/2$, the mean value theorem gives
\begin{align}
  \frac{\lambda}{2} \leq h(\lambda) \leq \sqrt2\,\lambda.\label{eq:h-linear}
\end{align}
Since $\lambda_T^\star\to0$ by \cref{lem:local-equation}, this bound applies at the root for all sufficiently large $T$.

We first show that
\begin{align}
  \mathfrak d_T \lesssim \lambda_T^\star+\psi_T(\lambda_T^\star).\label{eq:d-upper-by-root}
\end{align}
For the gap term in $\mathfrak d_T$, \eqref{eq:psi-root} and \eqref{eq:h-linear} give $\abs{1-r_T}\leq\psi_T(\lambda_T^\star)+\sqrt2\lambda_T^\star$. For the other two terms, use $\beta_T\leq\psi_T(\lambda_T^\star)\sqrt{\lambda_T^\star}$ and $\gamma_T\leq\psi_T(\lambda_T^\star)\lambda_T^\star$. Weighted AM--GM then gives
\begin{align*}
  \beta_T^{2/3}&\leq \frac23\psi_T(\lambda_T^\star)+\frac13\lambda_T^\star,\\
  \gamma_T^{1/2}&\leq \frac12\psi_T(\lambda_T^\star)+\frac12\lambda_T^\star,
\end{align*}
and \eqref{eq:d-upper-by-root} follows.

For the reverse inequality, suppose $\lambda_T^\star\geq4\mathfrak d_T$. Since $\beta_T\leq\mathfrak d_T^{3/2}$ and $\gamma_T\leq\mathfrak d_T^2$, we obtain
\begin{align*}
 \psi_T(\lambda_T^\star) &\leq\frac{\mathfrak d_T^{3/2}}{\sqrt{4\mathfrak d_T}} \vee\frac{\mathfrak d_T^2}{4\mathfrak d_T} =\frac{\mathfrak d_T}{2},\\
 \psi_T(\lambda_T^\star) &=(1-r_T)+h(\lambda_T^\star) \geq-\mathfrak d_T+\frac{\lambda_T^\star}{2} \geq\mathfrak d_T.
\end{align*}
This is impossible because $\mathfrak d_T>0$. Hence
\begin{align*}
 \lambda_T^\star<4\mathfrak d_T,\qquad \psi_T(\lambda_T^\star) \leq\mathfrak d_T+\sqrt2\lambda_T^\star \leq(1+4\sqrt2)\mathfrak d_T.
\end{align*}
Together with \eqref{eq:d-upper-by-root}, these bounds prove \eqref{eq:root-scale}. The bonus bound follows from $B_1(n_{1,T}^\star)=\Delta_T^{\mathrm{crit}}\psi_T(\lambda_T^\star)$. This completes the proof.
\end{proof}

\subsubsection{Separation and localization}

\begin{proof}[Proof of \cref{prop:critical-behavior}]
Since $\mathfrak d_T\to0$ implies \eqref{eq:singular-regime}, \cref{lem:local-equation} already gives $\lambda_T^\star\to0$.

\paragraph{Deterministic separation.}
For fixed $\kappa>0$, we have $\kappa\lambda_T^\star\to0$. Together with $\gamma_T\to0$, this means that the arm-2 bonus at count $T(1-\kappa\lambda_T^\star)$ uses its standard-deviation branch for all large $T$. Subtracting the root identity therefore gives
\begin{align*}
  \mathcal H_T(\kappa\lambda_T^\star) &=h(\kappa\lambda_T^\star)-h(\lambda_T^\star) +\psi_T(\lambda_T^\star)-\psi_T(\kappa\lambda_T^\star).
\end{align*}
Comparing the two branches of $\psi_T$ gives
\begin{align*}
  \min\{\kappa^{-1/2},\kappa^{-1}\}\psi_T(\lambda_T^\star) \leq\psi_T(\kappa\lambda_T^\star) \leq\max\{\kappa^{-1/2},\kappa^{-1}\}\psi_T(\lambda_T^\star).
\end{align*}
For fixed $\kappa\neq1$, this gives
\begin{align*}
 |1-\kappa^{-1/2}|\psi_T(\lambda_T^\star) &\leq|\psi_T(\kappa\lambda_T^\star)-\psi_T(\lambda_T^\star)| \leq|1-\kappa^{-1}|\psi_T(\lambda_T^\star).
\end{align*}
For all large $T$, both $\lambda_T^\star$ and $\kappa\lambda_T^\star$ lie in $(0,1/2]$. Since $1/2\leq h'(\lambda)\leq\sqrt2$ on this interval,
\begin{align*}
 \frac{|\kappa-1|}{2}\lambda_T^\star &\leq|h(\kappa\lambda_T^\star)-h(\lambda_T^\star)| \leq\sqrt2|\kappa-1|\lambda_T^\star.
\end{align*}
The two differences in $\mathcal H_T(\kappa\lambda_T^\star)$ have the same sign. Thus \cref{lem:root-scale} gives
\begin{align*}
  \begin{cases} \mathcal H_T(\kappa\lambda_T^\star),&\kappa>1,\\
  -\mathcal H_T(\kappa\lambda_T^\star),&0<\kappa<1 \end{cases} \asymp_\kappa\lambda_T^\star+\psi_T(\lambda_T^\star) \asymp_\kappa\mathfrak d_T.
\end{align*}
Taking $\kappa=1+\eps$ and $\kappa=1-\eps$ and multiplying by $\Delta_T^{\mathrm{crit}}$, we obtain \eqref{eq:upper-nearby-count}--\eqref{eq:lower-nearby-count}.

\paragraph{Stochastic localization.}
By \cref{lem:arm2-stability} and $n_{1,T}^{\star}/T\to0$,
\begin{align*}
 \frac{n_{1,T}}{T} =1-\frac{n_{2,T}}{n_{2,T}^{\star}} \frac{n_{2,T}^{\star}}{T}\pto0.
\end{align*}
This proves \eqref{eq:critical-localization-main}.
\end{proof}

\subsection{The local noise comparison}\label{app:local-noise}
Under $\mathfrak d_T\to0$, let $s_T:=\sigma_2/\sqrt T$. Let $k_T$ be a deterministic integer count such that $k_T\asymp n_{1,T}^\star$. The mean and bonus estimates in the proof of \cref{lem:deterministic-count}, together with \eqref{eq:bonus-scaling} and \cref{lem:root-scale}, give
\begin{align*}
  \abs{\bar X_{1,k_T}-\mu_1}+\abs{\hat B_1(k_T)-B_1(k_T)} =\Op \left(\frac{B_1(k_T)}{\sqrt{c_T}}\right) =\Op(s_T\mathfrak d_T)=\op(s_T).
\end{align*}
For any fixed $\alpha\in(0,1)$, \cref{lem:tail-moments} applied with $\ceil{\alpha T}$ in place of $k_T$ gives
\begin{align}
  \sup_{\ceil{\alpha T}\leq n\leq T}\abs{\hat B_2(n)-B_2(n)} \leq\sqrt{\frac{c_T}{\ceil{\alpha T}}} \sup_{n\geq\ceil{\alpha T}}\abs{\hat\sigma_{2,n}-\sigma_2}=\Op \left(\frac{\sqrt{c_T}}{T}\right) =\Op(s_T\gamma_T)=\op(s_T).\label{eq:critical-linear-bonus-error}
\end{align}
Thus, at these counts, only the arm-2 mean error may contribute at order $s_T$.

\subsection{A local terminal balance}\label{app:local-terminal-balance}

Suppose $\mathfrak d_T\to0$, and write
\begin{align*}
 s_T=\frac{\sigma_2}{\sqrt T},\qquad \bar{\mathfrak d}_T=\mathfrak d_T+c_T^{-1/2},\qquad \Lambda_T=\Delta_T^{\mathrm{crit}}\bar{\mathfrak d}_T.
\end{align*}
Then $\bar{\mathfrak d}_T\to0$, $T\bar{\mathfrak d}_T\to\infty$, and $s_T\leq \Lambda_T$.

\begin{lemma}\label{lem:local-terminal-balance}
Under $\mathfrak d_T\to0$,
\begin{align}
 B_1(N_T)&=\Op(\Lambda_T),\label{eq:local-terminal-bonus}\\
 H_T(N_T)+s_TY_T&=\op(\Lambda_T).\label{eq:local-terminal-balance}
\end{align}
\end{lemma}
\begin{proof}
\textit{The arm-1 bonus.} Set $k_T=\floor{T\bar{\mathfrak d}_T}$ and $k'_T=T-k_T-1$. Then $k_T\to\infty$ and $k'_T\sim T$, so all the counts below are positive for all large $T$. Since $\beta_T\leq\mathfrak d_T^{3/2}$, $\gamma_T\leq\mathfrak d_T^2$, and $k_T/T\sim \bar{\mathfrak d}_T$, we have
\begin{align*}
 \frac{B_1(k_T)}{\Delta_T^{\mathrm{crit}}} &=\frac{\beta_T}{\sqrt{k_T/T}}\vee\frac{\gamma_T}{k_T/T} \lesssim\frac{\mathfrak d_T^{3/2}}{\sqrt{\bar{\mathfrak d}_T}} \vee\frac{\mathfrak d_T^2}{\bar{\mathfrak d}_T} \leq \bar{\mathfrak d}_T.
\end{align*}
Also, $k'_T/T=1-\bar{\mathfrak d}_T+O(T^{-1})$ and $\gamma_T\to0$, so the arm-2 standard-deviation branch is active at $k'_T$ for all large $T$. Thus
\begin{align*}
 B_2(k'_T)-\Delta &=\Delta_T^{\mathrm{crit}} \{(k'_T/T)^{-1/2}-r_T\}=O(\Lambda_T),\\
 \UCB_2(k'_T)-\mu_1 &=B_2(k'_T)-\Delta+\Op(B_2(k'_T)/\sqrt{c_T}) =\Op(\Lambda_T),
\end{align*}
where the second line uses \cref{lem:deterministic-count} and $B_2(k'_T)/\sqrt{c_T}=O(s_T)$.

On $\{N_T\geq k_T\}$, monotonicity gives $B_1(N_T)\leq B_1(k_T)$. On $\{N_T<k_T\}$, arm 2 reaches $T-k_T$ pulls. Just before that pull, its count is $k'_T$ and the arm-1 count $J_T$ is at most $N_T$. On $\cE_T^{\rm unif}$,
\begin{align*}
 \mu_1+(1-\eta_T)B_1(N_T) \leq\mu_1+(1-\eta_T)B_1(J_T) \leq\UCB_1(J_T)\leq\UCB_2(k'_T).
\end{align*}
Combining the two cases gives, on this event,
\begin{align*}
 B_1(N_T) \leq B_1(k_T) +\frac{\{\UCB_2(k'_T)-\mu_1\}_+}{1-\eta_T} =\Op(\Lambda_T).
\end{align*}
This proves \eqref{eq:local-terminal-bonus}.

\textit{The two last-pull comparisons.} By \cref{prop:critical-behavior}, $N_T/T\pto0$. For any random positive integer $K_T\leq T$ with $K_T/T\pto1$, \cref{lem:deterministic-window} and $\E Y_T^2=1$ give
\begin{align}
 \frac{\bar X_{2,K_T}-\mu_2}{s_T} =\frac{T}{K_T}\left\{Y_T+\frac{S_{2,K_T}-S_{2,T}}{\sigma_2\sqrt T}\right\} =Y_T+\op(1).\label{eq:near-terminal-arm2-mean}
\end{align}
The bonus bound \eqref{eq:critical-linear-bonus-error} also applies at $K_T$. Indeed, for every $\eps>0$,
\begin{align*}
 &\Prob\{|\hat B_2(K_T)-B_2(K_T)|>\eps s_T\}\leq\Prob(K_T<T/2) +\Prob\left\{\sup_{\ceil{T/2}\leq n\leq T} |\hat B_2(n)-B_2(n)|>\eps s_T\right\}\longrightarrow0.
\end{align*}
Thus, writing
\begin{align*}
 e_{2,T}(k)=\UCB_2(k)-\mu_2-B_2(k)-s_TY_T,
\end{align*}
we have $e_{2,T}(K_T)=\op(s_T)$ whenever $K_T/T\pto1$.

By \eqref{eq:terminal-count-lower}, the last pull of each arm occurs after initialization for all large $T$. Just before the last arm-1 pull, the counts are $(N_T-1,K_T)$, where $1\leq K_T\leq T-N_T$. On $\cE_T^{\rm unif}$, \eqref{eq:bonus-one-step} gives
\begin{align*}
 e_{1,T}&:=\UCB_1(N_T-1)-\mu_1-B_1(N_T),\\
 |e_{1,T}|&\leq\eta_T B_1(N_T-1)+\frac{B_1(N_T)}{N_T-1} =\op(\Lambda_T).
\end{align*}
Selection of arm 1 therefore requires
\begin{align*}
 (1-\eta_T)\sigma_2\sqrt{c_T/K_T} &\leq\Delta+B_1(N_T)+e_{1,T},\\
 1\leq\sqrt{T/K_T} &\leq\frac{r_T+\Op(\bar{\mathfrak d}_T)}{1-\eta_T}\pto1.
\end{align*}
Hence $K_T/T\pto1$ and $e_{2,T}(K_T)=\op(s_T)$. The arm-1 selection inequality, together with $B_2(K_T)\geq B_2(T-N_T)$, gives
\begin{align*}
 H_T(N_T)+s_TY_T &\leq e_{1,T}-e_{2,T}(K_T),\\
 \{H_T(N_T)+s_TY_T\}_+ &\leq|e_{1,T}|+|e_{2,T}(K_T)|=\op(\Lambda_T).
\end{align*}

Just before the last arm-2 pull, the counts are $(J_T,K'_T)$, where $J_T\leq N_T$ and $K'_T=T-N_T-1$. Then $K'_T/T\pto1$ and $e_{2,T}(K'_T)=\op(s_T)$. Also,
\begin{align*}
 \frac{B_2(T-N_T)}{\Delta_T^{\mathrm{crit}}} &=\sqrt{\frac{T}{T-N_T}}\vee\frac{\gamma_T T}{T-N_T}\pto1,\\
 0\leq B_2(K'_T)-B_2(T-N_T) &\leq\frac{B_2(T-N_T)}{T-N_T-1} =\Op(\Delta_T^{\mathrm{crit}}/T)=\op(\Lambda_T).
\end{align*}
On $\cE_T^{\rm unif}$, selection of arm 2 implies
\begin{align*}
 \mu_2+B_2(K'_T)+s_TY_T+e_{2,T}(K'_T) &\geq\UCB_1(J_T)\geq\mu_1+(1-\eta_T)B_1(N_T).
\end{align*}
Consequently,
\begin{align*}
 \{-H_T(N_T)-s_TY_T\}_+ \leq\eta_T B_1(N_T)+B_2(K'_T)-B_2(T-N_T) +|e_{2,T}(K'_T)| =\op(\Lambda_T).
\end{align*}
Both comparisons hold on an event whose probability tends to one. Together they prove \eqref{eq:local-terminal-balance}.
\end{proof}

\paragraph{Critical subsequences.}\label{app:critical-bonus}
We record the bounds and Gaussian limit used below in the proofs of \cref{prop:critical-count,thm:fixed-target}.

Along any critical subsequence, $\bar{\mathfrak d}_T=O(c_T^{-1/2})$ and $\Lambda_T=O(s_T)$, so \cref{lem:local-terminal-balance} gives
\begin{align}
 B_1(N_T)&=\Op(s_T).\label{eq:critical-terminal-bonus}
\end{align}
On the separate-error event in \eqref{eq:separate-index-errors},
\begin{align}
 |\bar X_{1,N_T}-\mu_1| &\leq\varepsilon_T^{\mathrm{idx}}B_1(N_T)=\op(s_T),\notag\\
 \hat s_{1,T}^{\,2} &\leq\frac{\hat B_1(N_T)^2}{c_T}\leq\frac{(1+\sqrt{2\varepsilon_T^{\mathrm{idx}}})^2B_1(N_T)^2}{c_T}
 =\Op(s_T^2/c_T)=\op(s_T^2).\label{eq:critical-terminal-errors}
\end{align}
Since $\varepsilon_T^{\mathrm{idx}}\to0$ and this event has probability tending to one, both terminal error bounds hold in probability.
Since $g_T$ is bounded on this subsequence,
\begin{align}
 T\sigma_2^2=\frac{c_T^3}{g_T^2}\gtrsim c_T^3\to\infty. \label{eq:sigma-lower-critical}
\end{align}
The independent centered increments $(X_{2,i}-\mu_2)/(\sigma_2\sqrt T)$ have total variance one and are bounded by $(T\sigma_2^2)^{-1/2}\to0$. Thus \cref{lem:martingale-clt} and \cref{prop:critical-behavior} give
\begin{align}
 N_T/T\pto0,\qquad Y_T\dto N(0,1). \label{eq:critical-inference-preliminaries}
\end{align}

\subsection{Proof of Proposition~\ref{prop:critical-count}}\label{app:critical-count}
In this subsection, we work along a critical subsequence satisfying \eqref{eq:critical-coordinate-limits} and write $v_T=T/\sqrt{c_T}$.

\subsubsection{A bound for the root function}
\begin{lemma}\label{lem:root-continuity}
For $\xi,\widetilde\xi\in\R$ and $b,g,\widetilde b,\widetilde g\geq0$, let $\bar\chi$ and $\widetilde\chi$ be the root functions defined in \cref{sec:critical-count} for parameters $(\xi,b,g)$ and $(\widetilde\xi,\widetilde b,\widetilde g)$, respectively. Then
\begin{align}
 \sup_{y\in\R}|\bar\chi(y)-\widetilde\chi(y)| \leq 2|\xi-\widetilde\xi| +\left\{(2|b-\widetilde b|)^{2/3} \vee\sqrt{2|g-\widetilde g|}\right\}. \label{eq:root-parameter-bound}
\end{align}
The function $\bar\chi$ is continuous and nonincreasing in $y$, and strictly decreasing when $b+g>0$.
\end{lemma}
\begin{proof}
For $z>0$, set
\begin{align*}
 F(z,y)=\xi+y+z/2-(b/\sqrt z\vee g/z),
\end{align*}
and define $\widetilde F$ with the other parameters. For $z_2>z_1>0$,
\begin{align}
 F(z_2,y)-F(z_1,y)\geq(z_2-z_1)/2. \label{eq:root-increment}
\end{align}
Fix $y$ and suppose $x=\bar\chi(y)>\widetilde x=\widetilde\chi(y)$. Write $d=x-\widetilde x>0$ and
\begin{align*}
 \delta_\xi=|\xi-\widetilde\xi|,\qquad \delta_b=|b-\widetilde b|,\qquad \delta_g=|g-\widetilde g|.
\end{align*}
If $\widetilde x>0$, then $\widetilde F(x,y)\geq d/2$ by \eqref{eq:root-increment}. If $\widetilde x=0$, then $\widetilde b=\widetilde g=0$ and $\widetilde\xi+y\geq0$, so the same bound holds. Since $F(x,y)=0$ and $x\geq d$,
\begin{align*}
 \frac d2 \leq\widetilde F(x,y)-F(x,y)\leq \delta_\xi+\left(\frac{\delta_b}{\sqrt x}\vee\frac{\delta_g}{x}\right) \leq \delta_\xi+\left(\frac{\delta_b}{\sqrt d}\vee\frac{\delta_g}{d}\right).
\end{align*}
If $d'=d-2\delta_\xi>0$, this gives
\begin{align*}
 \frac{d'}2\leq\frac{\delta_b}{\sqrt{d'}}\vee\frac{\delta_g}{d'} \quad\Longrightarrow\quad d'\leq(2\delta_b)^{2/3}\vee\sqrt{2\delta_g}.
\end{align*}
Thus $d\leq2\delta_\xi+\{(2\delta_b)^{2/3}\vee\sqrt{2\delta_g}\}$. Interchanging the two parameter triples proves that bound.

If $b+g>0$ and $y_1<y_2$, both roots are positive and
\begin{align*}
 F(\bar\chi(y_1),y_2)=y_2-y_1>0=F(\bar\chi(y_2),y_2).
\end{align*}
Strict monotonicity of $F(\cdot,y_2)$ gives $\bar\chi(y_1)>\bar\chi(y_2)$. If $b=g=0$, monotonicity follows from $\bar\chi(y)=2(-\xi-y)_+$. Continuity follows from \eqref{eq:root-parameter-bound} by varying $\xi$, since a shift in $y$ is equivalent to the same shift in $\xi$.
\end{proof}

\subsubsection{Completion of the proof}
\begin{proof}
Along a critical subsequence, \cref{lem:local-terminal-balance} gives
\begin{align}
 H_T(N_T)/s_T+Y_T\pto0.\label{eq:terminal-stochastic-balance}
\end{align}
Recall that $\chi_T=N_T/v_T$. With probability tending to one, the arm-2 standard-deviation branch is active at $T-N_T$. On this event,
\begin{align*}
 0\leq \chi_T/2 \leq\sqrt{c_T}\{(1-N_T/T)^{-1/2}-1\}=B_1(N_T)/s_T-\xi_T-Y_T+\op(1)=\Op(1),
\end{align*}
by \eqref{eq:critical-terminal-bonus}. Thus $\chi_T=\Op(1)$. For $0\leq z\leq1/2$,
\begin{align*}
 |(1-z)^{-1/2}-1-z/2|\leq C z^2,
\end{align*}
so the rescaled Taylor remainder is at most $C\chi_T^2/\sqrt{c_T}=\op(1)$ with probability tending to one. Since $B_1(N_T)/s_T=b_T/\sqrt{\chi_T}\vee g_T/\chi_T$, we obtain
\begin{align*}
 F_T(\chi_T,Y_T)=\op(1),\qquad F_T(x,y):=\xi_T+y+x/2-(b_T/\sqrt x\vee g_T/x).
\end{align*}
Let $\bar\chi_T(y)$ be its unique positive root; $g_T>0$ ensures existence. By \eqref{eq:root-increment},
\begin{align}
 |\chi_T-\bar\chi_T(Y_T)|\leq2|F_T(\chi_T,Y_T)|=\op(1). \label{eq:count-random-root-approximation}
\end{align}
By \eqref{eq:critical-coordinate-limits} and \cref{lem:root-continuity},
\begin{align*}
 \sup_y|\bar\chi_T(y)-\bar\chi(y)| &\leq2|\xi_T-\xi| +\left\{(2|b_T-b|)^{2/3}\vee\sqrt{2|g_T-g|}\right\} \longrightarrow0.
\end{align*}
Together with \eqref{eq:count-random-root-approximation}, this proves \eqref{eq:critical-count-approximation-main}. Since $\bar\chi$ is continuous and $Y_T\dto G_2$, it also proves \eqref{eq:critical-count-limit}.

If $b+g>0$, the root is positive and strictly decreasing in $y$, so $\bar\chi(G_2)$ is positive and nonconstant. If $b=g=0$, then
\begin{align*}
 \Prob\{\bar\chi(G_2)=0\}=\Phi(\xi)\in(0,1),\qquad \bar\chi(G_2)=2(-\xi-G_2)>0\quad\text{on }\{G_2<-\xi\}.
\end{align*}
This proves the remaining claims.
\end{proof}

\subsection{Proof of Theorem~\ref{thm:sharp}}\label{app:sharp-proof}

\subsubsection{Sufficiency away from the local regime}
\begin{lemma}\label{lem:noncritical-stability}
If $\liminf_{T\to\infty}\mathfrak d_T>0$, then $n_{1,T}/n_{1,T}^\star\xrightarrow{p}1$.
\end{lemma}
\begin{proof}
Write $n_a=n_{a,T}$, $n_a^\star=n_{a,T}^{\star}$, and
\begin{align*}
 \lambda_T^\star=\frac{n_1^\star}{T},\qquad \vartheta_T=\frac{B_1(n_1^\star)}{B_2(n_2^\star)}\in(0,1).
\end{align*}
From any subsequence, choose a further one on which $\lambda_T^\star\to\lambda\in[0,1]$ and $\vartheta_T\to\vartheta\in[0,1]$. If $\lambda>0$, \cref{lem:arm2-stability} gives
\begin{align*}
 \frac{n_1}{n_1^\star}-1 =-\frac{n_2^\star}{n_1^\star}\left(\frac{n_2}{n_2^\star}-1\right)\pto0.
\end{align*}
If $\vartheta>0$, \cref{lem:arm2-stability} and \eqref{eq:bonus-scaling} imply $B_2(n_2)/B_2(n_2^\star)\pto1$. Using \eqref{eq:terminal-relative-balance},
\begin{align*}
 \frac{B_1(n_1)}{B_2(n_2^\star)} &=\frac{B_2(n_2)}{B_2(n_2^\star)\mathcal R_T} -\frac{\Delta}{B_2(n_2^\star)}\\
 &=1+\op(1)-(1-\vartheta_T)=\vartheta_T+\op(1).
\end{align*}
Dividing by $\vartheta_T\to\vartheta>0$ gives $B_1(n_1)/B_1(n_1^\star)\pto1$. For every fixed $\eps\in(0,1)$, \eqref{eq:bonus-scaling} gives
\begin{align*}
 n_1\geq(1+\eps)n_1^\star &\quad\Longrightarrow\quad \frac{B_1(n_1)}{B_1(n_1^\star)}\leq(1+\eps)^{-1/2},\\
 n_1\leq(1-\eps)n_1^\star &\quad\Longrightarrow\quad \frac{B_1(n_1)}{B_1(n_1^\star)}\geq(1-\eps)^{-1/2}.
\end{align*}
Both events have probability tending to zero, so $n_1/n_1^\star\pto1$.

It remains to exclude $\lambda=\vartheta=0$. If $B_2(n_2^\star)=c_T/n_2^\star$ along an infinite subsequence, then
\begin{align*}
 \vartheta_T\geq\frac{c_T/n_1^\star}{c_T/n_2^\star} =\frac{1-\lambda_T^\star}{\lambda_T^\star}\longrightarrow\infty,
\end{align*}
a contradiction. Thus the arm-2 standard-deviation branch is active eventually, and the balancing equation gives
\begin{align*}
 r_T=\frac{1-\vartheta_T}{\sqrt{1-\lambda_T^\star}}\to1,\qquad
 \beta_T \leq\frac{\vartheta_T\sqrt{\lambda_T^\star}} {\sqrt{1-\lambda_T^\star}}\to0,\qquad \gamma_T\leq\frac{\vartheta_T\lambda_T^\star} {\sqrt{1-\lambda_T^\star}}\to0.
\end{align*}
This would imply $\mathfrak d_T\to0$. The two preceding cases therefore cover every subsequence, proving the assertion.
\end{proof}

\subsubsection{Sufficiency in the local regime}\label{app:sufficiency}
\begin{proposition}\label{prop:local-sufficiency}
Suppose
\begin{align}
  \mathfrak d_T\to0, \qquad \sqrt{c_T}\,\mathfrak d_T\to\infty.\label{eq:local-suff-assumptions}
\end{align}
Then $n_{1,T}/n_{1,T}^\star\xrightarrow{p}1$.
\end{proposition}
\begin{proof}
Under \eqref{eq:local-suff-assumptions}, $\bar{\mathfrak d}_T\sim\mathfrak d_T$ and
\begin{align*}
 \frac{s_T}{\Delta_T^{\mathrm{crit}}\mathfrak d_T} =\frac{1}{\sqrt{c_T}\mathfrak d_T}\to0.
\end{align*}
Since $Y_T=\Op(1)$, \cref{lem:local-terminal-balance} yields
\begin{align*}
 \frac{H_T(n_{1,T})}{\Delta_T^{\mathrm{crit}}\mathfrak d_T} =\op(1).
\end{align*}
Fix $\eps\in(0,1/2)$. By \cref{prop:critical-behavior} and the monotonicity of $H_T$, for some $c_\eps>0$ and all large $T$,
\begin{align*}
 n_{1,T}\geq(1+\eps)n_{1,T}^{\star} &\quad\Longrightarrow\quad H_T(n_{1,T})\geq c_\eps\Delta_T^{\mathrm{crit}}\mathfrak d_T,\\
 n_{1,T}\leq(1-\eps)n_{1,T}^{\star} &\quad\Longrightarrow\quad H_T(n_{1,T})\leq-c_\eps\Delta_T^{\mathrm{crit}}\mathfrak d_T.
\end{align*}
Both events have probability tending to zero.
\end{proof}

\begin{proof}[Proof of sufficiency in \cref{thm:sharp}]
Under \eqref{eq:sharp-condition}, every subsequence has a further subsequence on which either $\inf_T\mathfrak d_T>0$ or $\mathfrak d_T\to0$. Apply \cref{lem:noncritical-stability} in the first case and \cref{prop:local-sufficiency} in the second. This proves \eqref{eq:optimal-root-stability} along the full sequence. The arm-2 assertion is \cref{lem:arm2-stability}.
\end{proof}

\subsubsection{Necessity}\label{app:necessity}
\begin{proof}[Proof of necessity in \cref{thm:sharp}]
If \eqref{eq:sharp-condition} fails, choose a critical subsequence and then a further one satisfying \eqref{eq:critical-coordinate-limits}. By \cref{prop:critical-count}, along this subsequence,
\begin{align*}
 \chi_T=\frac{n_{1,T}}{v_T}\dto \bar\chi(G_2),\qquad v_T=\frac{T}{\sqrt{c_T}},
\end{align*}
where the limit is finite and nonconstant. Suppose that a positive deterministic sequence $m_T$ satisfies $n_{1,T}/m_T\pto1$. Pass to a further subsequence on which $m_T/v_T\to \bar m\in[0,\infty]$. Then
\begin{align*}
 \bar m<\infty&\quad\Longrightarrow\quad \chi_T=\frac{n_{1,T}}{m_T}\frac{m_T}{v_T}\pto \bar m,\\
 \bar m=\infty&\quad\Longrightarrow\quad \chi_T\pto\infty.
\end{align*}
The first case contradicts the nonconstant limit. In the second case, for each fixed $K>0$ and all large $T$,
\begin{align*}
 \Prob(\chi_T\leq K) \leq\Prob(n_{1,T}/m_T\leq1/2)\longrightarrow0,
\end{align*}
since $m_T/v_T>2K$. This contradicts tightness of $\chi_T$. No such deterministic normalizer exists.
\end{proof}

\section{Proofs for Section~\ref{sec:inference}: Wald inference}\label{app:inference}

\subsection{Proof of Proposition~\ref{prop:cross-arm}}\label{app:cross-arm}
\begin{proof}
The auxiliary policy uses only the arm-2 rewards and deterministic population quantities. Thus $n^{\circ}_{1,T}$ is $\cF_{2,T}$-measurable and independent of the arm-1 reward sequence. Let $\eta_T$ and $\cE_T^{\rm unif}$ be as in \eqref{eq:eta-ucbv}--\eqref{eq:ucbv-upper-lower}. We prove the pathwise bound in \eqref{eq:cross-arm-bound} on $\cE_T^{\rm unif}$ for all sufficiently large $T$.
By \eqref{eq:bonus-scaling}, $B_1(\kappa n)\leq \kappa^{-1/2}B_1(n)$ for $\kappa\geq1$. Fix a path in $\cE_T^{\rm unif}$ and write $N=N_T$, $M=n^{\circ}_{1,T}$, and $\eta=\eta_T<1$.

Suppose $N>(1+\eta)^2M+1$. Just before the original policy's last arm-1 pull, its counts are $(N-1,j)$, where $N-1>(1+\eta)^2M$ and $j\leq T-N<T-M$. The auxiliary policy ends with $T-M$ arm-2 pulls, so it must select arm 2 from count $j$, with some arm-1 count $i\leq M$. Thus
\begin{align*}
 \UCB_2(j)&\geq\mu_1+B_1(i)\geq\mu_1+B_1(M),\\
 \UCB_1(N-1)&\leq\mu_1+(1+\eta)B_1(N-1)<\mu_1+B_1(M)\leq\UCB_2(j).
\end{align*}
This contradicts the original policy's selection of arm 1. Here the two policies have the same $\UCB_2(j)$ because the bonus uses the fixed horizon, rather than the current round.

For the lower bound, suppose $M>(1-\eta)^{-2}N+1$. Just before the auxiliary policy's last arm-1 pull, its counts are $(M-1,j)$, with $M-1>(1-\eta)^{-2}N$ and $j\leq T-M<T-N$. The original policy must select arm 2 from count $j$, with some arm-1 count $i\leq N$. Consequently,
\begin{align*}
 \UCB_2(j)&\geq\UCB_1(i) \geq\mu_1+(1-\eta)B_1(i) \geq\mu_1+(1-\eta)B_1(N),\\
 \UCB_1^\circ(M-1)&=\mu_1+B_1(M-1) <\mu_1+(1-\eta)B_1(N)\leq\UCB_2(j),
\end{align*}
contradicting the auxiliary policy's selection of arm 1. In both cases the assumed count inequality places the last pull after initialization, so $j\geq1$.

Finally, by \eqref{eq:terminal-count-lower}, $N_T>c_T/2$ on every path for all large $T$. The upper bound in \eqref{eq:cross-arm-bound} therefore gives $n^{\circ}_{1,T}\geq(N_T-1)/(1+\eta_T)^2\to\infty$ on $\cE_T^{\rm unif}$. Dividing \eqref{eq:cross-arm-bound} by $n^{\circ}_{1,T}$ and using $\eta_T\to0$ and $\Prob(\cE_T^{\rm unif})\to1$ then gives $N_T/n^{\circ}_{1,T}\pto1$. This completes the proof.
\end{proof}

\subsection{Proof of Theorem~\ref{thm:fixed-target}}\label{app:fixed-target}
\subsubsection{The arm-1 Gaussian limit}\label{app:conditional-gaussian}

\begin{lemma}\label{lem:cross-arm-gaussian}
Suppose $N_T\sigma_1^2\pto\infty$. For every fixed $t\in\R$ and every sequence of complex, $\cF_{2,T}$-measurable variables $\Upsilon_T$ with $|\Upsilon_T|\leq1$,
\begin{align}
 \E[\Upsilon_Te^{itZ_{1,T}}]-\E[\Upsilon_T]e^{-t^2/2}\longrightarrow0. \label{eq:cross-arm-mixing}
\end{align}
In particular, $Z_{1,T}\dto N(0,1)$.
\end{lemma}
\begin{proof}
By \cref{prop:cross-arm},
\begin{align*}
 N_T/n^{\circ}_{1,T}\pto1,\qquad n^{\circ}_{1,T}\sigma_1^2=(n^{\circ}_{1,T}/N_T)(N_T\sigma_1^2)\pto\infty.
\end{align*}
Use the enlarged filtration $\cG_{T,t}=\cF_{T,t}\vee\cF_{2,T}$ and define
\begin{align*}
 d_{T,t} =\frac{\one\{A_t=1\}(R_t-\mu_1)}{\sigma_1\sqrt{n^{\circ}_{1,T}}}, \qquad Z^{\circ}_{1,T}=\sum_{t=1}^Td_{T,t} =\frac{S_{1,N_T}}{\sigma_1\sqrt{n^{\circ}_{1,T}}}.
\end{align*}
Here $n^{\circ}_{1,T}$ is $\cG_{T,0}$-measurable, and $\sigma_1>0$ for all large $T$. Also, $n^{\circ}_{1,T}\geq1$ gives $|d_{T,t}|\leq1/\sigma_1$, so each increment is square-integrable.

To check the conditional moments, let $E_{t,k}=\{A_t=1,n_{1,t-1}=k\}$. For any bounded, $\cG_{T,t-1}$-measurable $V$, the variable $V\one_{E_{t,k}}$ is measurable with respect to $\cF_{2,T}\vee\sigma(X_{1,1},\ldots,X_{1,k})$: on $E_{t,k}$ the observed history uses only these rewards. Independence of the next arm-1 reward therefore gives, for $j=1,2$,
\begin{align*}
 \E[V\one\{A_t=1\}(R_t-\mu_1)^j] &=\sum_{k=0}^{t-1}\E[V\one_{E_{t,k}}] \E[(X_{1,k+1}-\mu_1)^j]\\
 &=\begin{cases} 0,&j=1,\\
 \sigma_1^2\E[V\one\{A_t=1\}],&j=2. \end{cases}
\end{align*}
Since $A_t$ is $\cG_{T,t-1}$-measurable, this proves
\begin{align*}
 \E[d_{T,t}\mid\cG_{T,t-1}]&=0,\\
 \sum_{t=1}^T\E[d_{T,t}^2\mid\cG_{T,t-1}] &=\frac{N_T}{n^{\circ}_{1,T}}\pto1.
\end{align*}
For every $\eps>0$, bounded rewards give
\begin{align*}
 \sum_{t=1}^T \E[d_{T,t}^2\one\{|d_{T,t}|>\eps\}\mid\cG_{T,t-1}]\leq\frac{N_T}{n^{\circ}_{1,T}} \one\{(n^{\circ}_{1,T}\sigma_1^2)^{-1/2}>\eps\}\pto0.
\end{align*}
By \cref{lem:martingale-clt}, $Z^{\circ}_{1,T}\dto N(0,1)$ and \eqref{eq:cross-arm-mixing} holds with $Z^{\circ}_{1,T}$ in place of $Z_{1,T}$. By \cref{lem:random-index-variance}, $\hat\sigma_{1,N_T}/\sigma_1\pto1$. On an event whose probability tends to one,
\begin{align*}
 Z_{1,T} =Z^{\circ}_{1,T}\sqrt{\frac{n^{\circ}_{1,T}}{N_T}} \frac{\sigma_1}{\hat\sigma_{1,N_T}} =Z^{\circ}_{1,T}+\op(1).
\end{align*}
Finally,
\begin{align*}
 \left|\E[\Upsilon_T(e^{itZ_{1,T}}-e^{itZ^{\circ}_{1,T}})]\right| \leq\E\{2\wedge(|t|\,|Z_{1,T}-Z^{\circ}_{1,T}|)\}\longrightarrow0.
\end{align*}
This proves \eqref{eq:cross-arm-mixing}.
\end{proof}

\subsubsection{Fixed targets}
\begin{proof}[Proof of \cref{thm:fixed-target}]
Fix $u\neq0$. If $u_2=0$, \cref{lem:cross-arm-gaussian} gives $W_T(u)=\operatorname{sgn}(u_1)Z_{1,T}\dto N(0,1)$. Suppose $u_2\neq0$. Every subsequence has a further subsequence on which either $\sqrt{c_T}\mathfrak d_T\to\infty$ or $\sqrt{c_T}\mathfrak d_T$ is bounded. In the first case, \cref{thm:sharp} and \eqref{eq:joint-effective-sizes} give $n_{a,T}/n_{a,T}^{\star}\pto1$ for $a=1,2$. Moreover,
\begin{align*}
 n_{a,T}^{\star}\sigma_a^2 =\frac{n_{a,T}^{\star}}{n_{a,T}}(n_{a,T}\sigma_a^2)\pto\infty.
\end{align*}
The left side is deterministic, so its divergence verifies the remaining condition of \cref{lem:stable-wald}.

In the critical case, \eqref{eq:critical-inference-preliminaries} gives $Y_T\dto N(0,1)$ without requiring convergence of the critical coordinates. By \eqref{eq:near-terminal-arm2-mean}, \eqref{eq:critical-terminal-errors}, and variance consistency,
\begin{align*}
 \frac{\bar X_{2,n_{2,T}}-\mu_2}{s_T}=Y_T+\op(1), \qquad \frac{\hat s_{2,T}}{s_T}\pto1,\\
 \frac{\bar X_{1,n_{1,T}}-\mu_1}{s_T}\pto0, \qquad \frac{\hat s_{1,T}^{\,2}}{s_T^2}\pto0, \qquad s_T=\frac{\sigma_2}{\sqrt T}.
\end{align*}
Therefore, on the event of positive empirical variances,
\begin{align*}
 W_T(u) =\frac{u_1(\bar X_{1,n_{1,T}}-\mu_1)/s_T +u_2(\bar X_{2,n_{2,T}}-\mu_2)/s_T} {\{u_1^2\hat s_{1,T}^{\,2}/s_T^2 +u_2^2\hat s_{2,T}^{\,2}/s_T^2\}^{1/2}} =\frac{u_2Y_T+\op(1)}{|u_2|\{1+\op(1)\}} \dto N(0,1).
\end{align*}
Every subsequence has a further subsequence with this limit, so convergence holds along the full sequence. For $z=z_{1-\alpha/2}$,
\begin{align*}
 \left|\Prob\{\theta_T(u)\in C_{T,\alpha}(u)\} -\Prob\{-z\leq W_T(u)\leq z\}\right| \leq\Prob\left\{\min_a\hat\sigma_{a,n_{a,T}}=0\right\}\to0.
\end{align*}
The Gaussian limit then gives coverage tending to $1-\alpha$.

\end{proof}

\subsection{Proof of Theorem~\ref{thm:uniform-wald}}\label{app:uniform-wald}
\subsubsection{The joint critical limit}
\begin{lemma}\label{lem:critical-joint-limit}
Assume \eqref{eq:joint-effective-sizes} and work along a critical subsequence satisfying \eqref{eq:critical-coordinate-limits}. Then
\begin{align}
 (Z_{1,T},Z_{2,T},\chi_T) \dto(G_1,G_2,\bar\chi(G_2)),\qquad G_1,G_2\overset{\rm iid}{\sim}N(0,1).\label{eq:joint-pivots-count}
\end{align}
\end{lemma}
\begin{proof}
By \cref{prop:critical-count}, the arm-2 replacement in \eqref{eq:near-terminal-arm2-mean}, and variance consistency,
\begin{align*}
 \chi_T-\bar\chi(Y_T)\pto0,\qquad Y_T\dto G_2,\qquad Z_{2,T}-Y_T\pto0.
\end{align*}
For fixed $t_1,t_2,t_3\in\R$, take $\Upsilon_T=\exp\{i(t_2Y_T+t_3\bar\chi(Y_T))\}$ in \eqref{eq:cross-arm-mixing}. This variable is $\cF_{2,T}$-measurable and has modulus one, so
\begin{align*}
 &\E\exp\{i(t_1Z_{1,T}+t_2Y_T+t_3\bar\chi(Y_T))\}\\
 &\quad=e^{-t_1^2/2}\E\exp\{i(t_2Y_T+t_3\bar\chi(Y_T))\}+o(1)\longrightarrow e^{-t_1^2/2} \E\exp\{i(t_2G_2+t_3\bar\chi(G_2))\}.
\end{align*}
Continuity of $\bar\chi$ justifies the last limit. The characteristic-function limit and the replacements for $Z_{2,T}$ and $\chi_T$ prove \eqref{eq:joint-pivots-count}.
\end{proof}

\subsubsection{Sufficiency}
\begin{proof}[Proof of sufficiency in \cref{thm:uniform-wald}]
If arm 1 is stable, \cref{thm:sharp} gives $n_{a,T}/n_{a,T}^{\star}\pto1$ for both arms. The effective sample sizes imply $n_{a,T}^{\star}\sigma_a^2\to\infty$ for $a=1,2$. Apply \cref{lem:stable-wald} with $m_{a,T}=n_{a,T}^{\star}$ to obtain $D_T\to0$. This completes the proof.
\end{proof}

\subsubsection{Necessity}
\begin{proof}[Proof of necessity in \cref{thm:uniform-wald}]
If arm 1 is unstable, \cref{thm:sharp} gives a critical subsequence. Its coordinates are bounded, so we pass to a further subsequence satisfying \eqref{eq:critical-coordinate-limits}. All limits in this proof are along that subsequence. Take $u_T=(1,u_{2,T})$ with $u_{2,T}=(\sigma_1/\sigma_2)c_T^{1/4}$. The effective-sample-size assumption ensures $\sigma_1>0$ eventually. By \cref{cor:critical-coefficients}, proved below from \cref{lem:critical-joint-limit},
\begin{align*}
 W_T(u_T)\dto W_1 =\frac{G_1+\sqrt{\bar\chi(G_2)}G_2}{\sqrt{1+\bar\chi(G_2)}}.
\end{align*}
The variable $W_1$ has negative expectation and hence is not standard normal. Its distribution function $\Phi_1$ is continuous by \eqref{eq:critical-wald-cdf}. Choose $z_0$ with $\Phi_1(z_0)\neq\Phi(z_0)$. Then
\begin{align*}
 \liminf_T D_T \geq\lim_T|\Prob\{W_T(u_T)\leq z_0\}-\Phi(z_0)|=|\Phi_1(z_0)-\Phi(z_0)|>0.
\end{align*}
Thus $D_T$ cannot tend to zero along the full sequence. This completes the proof.
\end{proof}

\subsection{Proof of Corollary~\ref{cor:critical-coefficients}}
\label{app:critical-coefficients}
\begin{proof}
Let $\cE_T^{+}=\{\min_a\hat\sigma_{a,n_{a,T}}>0\}$. Relative variance consistency and $N_T/T\pto0$ give $\Prob(\cE_T^{+})\to1$ and
\begin{align*}
 Q_T:=(1-N_T/T)^{-1/2} \frac{\hat\sigma_{2,n_{2,T}}/\sigma_2} {\hat\sigma_{1,n_{1,T}}/\sigma_1}\pto1.
\end{align*}
Here and below, ratios are set to zero off $\cE_T^{+}$. On $\cE_T^{+}$, the weight in \eqref{eq:moving-weight-wald} satisfies
\begin{align}
 w_T=u_{2,T}\frac{\hat s_{2,T}}{\hat s_{1,T}} =\tau_T\sqrt{\chi_T}\,Q_T.\label{eq:critical-relative-weight}
\end{align}
If $\tau_T\to\tau<\infty$, tightness of $\chi_T$ implies $w_T-\tau_T\sqrt{\chi_T}\pto0$. Apply \cref{lem:critical-joint-limit} and the continuous mapping $ (z_1,z_2,r)\mapsto(z_1+rz_2)/\sqrt{1+r^2}$ to obtain \eqref{eq:finite-coefficient-limit}. This map is continuous at $r=0$, so a zero atom in $\bar\chi(G_2)$ causes no problem for finite $\tau$.

When $b+g>0$ and $\tau_T\to\infty$, \eqref{eq:critical-relative-weight} gives $w_T/\tau_T\dto\sqrt{\bar\chi(G_2)}$, whose limit is strictly positive almost surely. It follows that $w_T\pto\infty$: for any $K>0$ and $\delta>0$,
\begin{align*}
 \limsup_T\Prob(w_T\leq K) &\leq\limsup_T\Prob(w_T/\tau_T\leq\delta)\leq\Prob(\sqrt{\bar\chi(G_2)}\leq\delta),
\end{align*}
and the last probability tends to zero as $\delta\downarrow0$. Tightness of the two studentized means now gives
\begin{align*}
 W_T(1,u_{2,T})-Z_{2,T} &=\frac{Z_{1,T}}{\sqrt{1+w_T^2}} +\left(\frac{w_T}{\sqrt{1+w_T^2}}-1\right) Z_{2,T}=\op(1)
\end{align*}
on $\cE_T^{+}$. This proves \eqref{eq:large-coefficient-limit}.

It remains to check that $W_\tau$ is not standard normal when $0<\tau<\infty$. Set $\omega_\tau(y)=\tau\sqrt{\bar\chi(y)}/\sqrt{1+\tau^2\bar\chi(y)}$. The function $\omega_\tau$ is bounded, continuous, nonincreasing, and nonconstant, including when $b=g=0$. Choose $y_-<y_+$ with $\delta_\omega:=\omega_\tau(y_-)-\omega_\tau(y_+)>0$, and let $G_2'$ be an independent copy of $G_2$. The product below is nonpositive everywhere; on $\{G_2\leq y_-,G_2'\geq y_+\}$ it is at most $-(y_+-y_-)\delta_\omega$. Hence
\begin{align*}
 2\E[G_2\omega_\tau(G_2)] &=\E[(G_2-G_2')(\omega_\tau(G_2)-\omega_\tau(G_2'))]\\
 &\leq-(y_+-y_-)\delta_\omega\Phi(y_-)\{1-\Phi(y_+)\}<0.
\end{align*}
Also, $|W_\tau|\leq|G_1|+|G_2|$, and independence gives
\begin{align*}
 \E W_\tau =\E \left[\frac{\E(G_1\mid G_2)}{\sqrt{1+\tau^2\bar\chi(G_2)}}\right] +\E[G_2\omega_\tau(G_2)]<0.
\end{align*}
Finally, conditioning on $G_2$ gives its distribution function
\begin{align}
 \Phi_\tau(z)=\E\Phi \left( z\sqrt{1+\tau^2\bar\chi(G_2)}-\tau\sqrt{\bar\chi(G_2)}G_2\right).\label{eq:critical-wald-cdf}
\end{align}
It is continuous by bounded convergence. For $\tau=0$, $W_0=G_1$ directly.
\end{proof}

\section{Proofs for Section~\ref{sec:comparison}: Comparison with canonical UCB}\label{app:interval-lengths}

\subsection{The canonical UCB counts}\label{app:canonical-counts}
Let $B_T^U(n)=\sqrt{c_T/(2n)}$. The function $m\mapsto B_T^U(m)-B_T^U(T-m)-\Delta$ is continuous and strictly decreasing on $(0,T)$, with endpoint limits $+\infty$ and $-\infty$. Its unique zero $n_{2,T}^{U\star}$ and the complementary count $n_{1,T}^{U\star}=T-n_{2,T}^{U\star}$ therefore satisfy
\begin{align}
 B_T^U(n_{2,T}^{U\star})=\Delta+B_T^U(n_{1,T}^{U\star}), \qquad n_{1,T}^{U\star}+n_{2,T}^{U\star}=T.\label{eq:length-ucb-balance}
\end{align}
Since $\Delta>0$, $n_{2,T}^{U\star}<n_{1,T}^{U\star}$ and $n_{1,T}^{U\star}>T/2$. Also, $\Delta\leq1$ gives $B_T^U(n_{2,T}^{U\star})\leq1+\sqrt{c_T/T}$, hence $n_{2,T}^{U\star}\gtrsim c_T\to\infty$. We claim that
\begin{align}
 n_{a,T}^{U}/n_{a,T}^{U\star}\pto1,\qquad a=1,2.\label{eq:length-ucb-stability}
\end{align}

Because the policy compares indices at random counts, we need a mean-error bound that holds at every count. Hoeffding's bound $\E e^{t(X-\E X)}\leq e^{t^2/8}$ makes $\exp\{tS_{a,k}-kt^2/8\}$ a nonnegative supermartingale. For $t>0$, Ville's inequality gives
\begin{align*}
 \Prob\left(\max_{k\leq2m}S_{a,k}\geq x\right) &\leq\exp\{-tx+mt^2/4\} =e^{-x^2/m},\qquad t=2x/m.
\end{align*}
Applying the same bound to $-S_{a,k}$ yields $\Prob(\max_{k\leq2m}|S_{a,k}|\geq x)\leq2e^{-x^2/m}$. On a block $m\leq n\leq2m$, a deviation $|\bar X_{a,n}-\mu_a|>\eta B_T^U(n)$ requires $|S_{a,n}|>\eta\sqrt{c_Tm/2}$. Summing the maximal bound over dyadic blocks and the two arms gives
\begin{align*}
 \Prob\left\{ \max_{a=1,2}\sup_{1\leq n\leq T} \frac{|\bar X_{a,n}-\mu_a|}{B_T^U(n)}>\eta \right\} \leq4(1+\log_2T)e^{-\eta^2c_T/2}.
\end{align*}
We can therefore choose a deterministic $\eta_T^U\to0$ for which the event
\begin{align*}
 \cE_T^U=\left\{ \mu_a+(1-\eta_T^U)B_T^U(n)\leq I_a^U(n) \leq\mu_a+(1+\eta_T^U)B_T^U(n) \text{ for both arms and all }n\leq T \right\}
\end{align*}
has probability tending to one. For example, take $\eta_T^U=\sqrt{8\log\log(e^eT)/c_T}$ for sufficiently large $T$.

Work on $\cE_T^U$ and fix $\eps\in(0,1/2)$, abbreviating $n_a^\star=n_{a,T}^{U\star}$, $n_a=n_{a,T}^{U}$, and $\eta=\eta_T^U$. Suppose first that $n_2\geq(1+\eps)n_2^\star$. At the last arm-2 pull, the arm-1 count is at most $T-n_2<n_1^\star$. Selection of arm 2 and \eqref{eq:length-ucb-balance} give
\begin{align*}
 (1+\eta)B_T^U(n_2-1) &\geq\Delta+(1-\eta)B_T^U(n_1^\star)\\
 &=B_T^U(n_2^\star)-\eta B_T^U(n_1^\star) \geq(1-\eta)B_T^U(n_2^\star).
\end{align*}
But $n_2^\star\to\infty$, and the assumed deviation implies $B_T^U(n_2-1)/B_T^U(n_2^\star)\leq(1+\eps)^{-1/2}+o(1)$. This contradicts the displayed inequality for large $T$.

For the opposite deviation, suppose $n_2\leq(1-\eps)n_2^\star$. Then $n_1-1\geq n_1^\star$ for large $T$. At the last arm-1 pull, the arm-2 count is at most $n_2$. Selection of arm 1 requires
\begin{align*}
 (1-\eta)B_T^U(n_2) &\leq\Delta+(1+\eta)B_T^U(n_1-1)\\
 &\leq\Delta+(1+\eta)B_T^U(n_1^\star) \leq(1+\eta)B_T^U(n_2^\star).
\end{align*}
Now $B_T^U(n_2)/B_T^U(n_2^\star)\geq(1-\eps)^{-1/2}$ gives the contradiction. The relevant last pull occurs after initialization in both cases: the assumed deviations, together with $n_2^\star\to\infty$ and $n_1^\star>T/2$, ensure this. Since $\Prob(\cE_T^U)\to1$, we have $n_2/n_2^\star\pto1$. For arm 1, use $n_1-n_1^\star=-(n_2-n_2^\star)$ and $n_2^\star/n_1^\star\leq1$ to obtain $n_1/n_1^\star\pto1$ as well.

\subsection{The comparison in Section~\ref{sec:comparison-fixed}}\label{app:comparison-fixed}
With fixed reward laws, $\sigma_1,\sigma_2>0$, and $\Delta>0$, the canonical balancing equation gives $n_{2,T}^{U\star}\leq c_T/(2\Delta^2)=O(c_T)$. It follows that $n_{1,T}^{U\star}\sim T$ and $B_T^U(n_{1,T}^{U\star})\to0$. Substitution into \eqref{eq:length-ucb-balance} then gives $n_{2,T}^{U\star}\sim c_T/(2\Delta^2)$. For the variance-aware rule, solving the balancing equation \eqref{eq:deterministic-root} for the arm-2 count gives
\begin{align}
 \frac{n_{2,T}^{\star}}{c_T} =\frac{\sigma_2^2}{\{\Delta+B_1(n_{1,T}^{\star})\}^2} \vee\frac1{\Delta+B_1(n_{1,T}^{\star})}.\label{eq:length-v-balance-inversion}
\end{align}
Here again $n_{2,T}^{\star}=O(c_T)$, $n_{1,T}^{\star}\sim T$, and $B_1(n_{1,T}^{\star})\to0$. Continuity of the maximum in \eqref{eq:length-v-balance-inversion} gives
\begin{align*}
 n_{2,T}^{\star}\sim c_T \left(\frac{\sigma_2^2}{\Delta^2}\vee\frac1\Delta\right) =\frac{c_T(\sigma_2^2\vee\Delta)}{\Delta^2}.
\end{align*}
The variance-aware counts are stable by \cref{thm:sharp}. This count approximation and \eqref{eq:length-ucb-stability} prove \eqref{eq:length-fixed-counts}, with $n_{1,T}^{\pi}/T\pto1$ for $\pi\in\{U,V\}$.

All four effective sample sizes diverge because the counts diverge and the variances are fixed and positive. Thus $\hat\sigma_{a,T}^{\pi}/\sigma_a\pto1$ (\cref{lem:random-index-variance}). Writing $(\hat s_{a,T}^{\pi})^2=(\hat\sigma_{a,T}^{\pi})^2/n_{a,T}^{\pi}$, we obtain
\begin{samepage}
\begin{align*}
 (\hat s_{1,T}^{\pi})^2 &=\{1+\op(1)\}\frac{\sigma_1^2}{T}, \qquad \pi\in\{U,V\},\\
 (\hat s_{2,T}^{U})^2 &=\{1+\op(1)\}\frac{2\sigma_2^2\Delta^2}{c_T}, \qquad (\hat s_{2,T}^{V})^2 =\{1+\op(1)\}\frac{\sigma_2^2\Delta^2} {c_T(\sigma_2^2\vee\Delta)}.
\end{align*}
\end{samepage}
Since $c_T/T\to0$, these expressions give
\begin{align*}
 \frac{(\hat s_{1,T}^{\pi})^2}{(\hat s_{2,T}^{\pi})^2} &=\Op(c_T/T)=\op(1),\qquad \pi\in\{U,V\},\\
 \left(\frac{L_{1,T}^{V}}{L_{1,T}^{U}}\right)^2 &=\frac{(\hat s_{1,T}^{V})^2}{(\hat s_{1,T}^{U})^2}\pto1,\\
 \left(\frac{L_{\Delta,T}^{V}}{L_{\Delta,T}^{U}}\right)^2 &=\frac{(\hat s_{1,T}^{V})^2+(\hat s_{2,T}^{V})^2} {(\hat s_{1,T}^{U})^2+(\hat s_{2,T}^{U})^2} \pto\frac{1}{2(\sigma_2^2\vee\Delta)}.
\end{align*}
Taking nonnegative square roots proves \eqref{eq:length-fixed-ratios}. For regret, the common gap cancels, and the two count limits give
\begin{align*}
 \frac{\mathrm{Reg}_T^V}{\mathrm{Reg}_T^U} =\frac{n_{2,T}^V}{n_{2,T}^U} =\{1+\op(1)\} \frac{c_T(\sigma_2^2\vee\Delta)/\Delta^2}{c_T/(2\Delta^2)} \pto2(\sigma_2^2\vee\Delta).
\end{align*}
This proves \eqref{eq:regret-fixed-ratio}.

\subsection{The comparison in Section~\ref{sec:comparison-near}}\label{app:comparison-near}
Assume \eqref{eq:comparison-near-cancellation}. We first determine the canonical UCB counts. Let $\lambda^U_T=n_{2,T}^{U\star}/T$. Multiplying \eqref{eq:length-ucb-balance} by $\sqrt{2T/c_T}$ gives
\begin{align}
 \frac1{\sqrt{\lambda^U_T}}-\frac1{\sqrt{1-\lambda^U_T}} =\Delta\sqrt{\frac{2T}{c_T}} =\sqrt2\,r_T\sigma_2.\label{eq:length-near-ucb-fraction}
\end{align}
Since $r_T\sigma_2\to \bar\sigma$ and the left side decreases continuously from $+\infty$ to $0$ as $\lambda^U_T$ ranges from $0$ to $1/2$, we have $\lambda^U_T\to\lambda^U_{\bar\sigma}\in(0,1/2]$, where $\lambda^U_{\bar\sigma}$ solves \eqref{eq:length-critical-fraction}. Applying \eqref{eq:length-ucb-stability}, we obtain
\begin{align}
 \frac{n_{1,T}^{U}}{T}\pto1-\lambda^U_{\bar\sigma}, \qquad \frac{n_{2,T}^{U}}{T}\pto\lambda^U_{\bar\sigma}.\label{eq:length-critical-u-counts}
\end{align}
The assumed effective sample sizes under the variance-aware rule imply $T\sigma_a^2\to\infty$, since $n_{a,T}^{V}\leq T$. Both canonical UCB counts are asymptotic to positive fractions of $T$, so their effective sample sizes also diverge. We may therefore use relative variance consistency under both policies.

For variance-aware UCB, $N_T=n_{1,T}^{V}$. By \cref{prop:critical-behavior}, $N_T/T\pto0$ and $n_{2,T}^{V}/T\pto1$. With $\bar{\mathfrak d}_T=\mathfrak d_T+c_T^{-1/2}\to0$ as in Appendix~\ref{app:local-terminal-balance}, \cref{lem:local-terminal-balance} and $B_1(N_T)\geq\sigma_1\sqrt{c_T/N_T}$ give
\begin{align}
 \frac{\sigma_1^2/N_T}{\sigma_2^2/T} \leq\left(\frac{B_1(N_T)}{\Delta_T^{\mathrm{crit}}}\right)^2 =\Op(\bar{\mathfrak d}_T^{\,2})=\op(1).\label{eq:length-near-arm1-negligible}
\end{align}
This bound holds throughout $\mathfrak d_T\to0$.

Relative variance consistency and \eqref{eq:length-near-arm1-negligible} give
\begin{align}
 (\hat s_{2,T}^{V})^2=\{1+\op(1)\}\frac{\sigma_2^2}{T}, \qquad (\hat s_{1,T}^{V})^2=\op(\sigma_2^2/T).\label{eq:length-critical-v-variances}
\end{align}
For canonical UCB, \eqref{eq:length-critical-u-counts} gives
\begin{align*}
 (\hat s_{1,T}^{U})^2 =\{1+\op(1)\}\frac{\sigma_1^2}{(1-\lambda^U_{\bar\sigma})T},\qquad (\hat s_{2,T}^{U})^2 =\{1+\op(1)\}\frac{\sigma_2^2}{\lambda^U_{\bar\sigma}T}.
\end{align*}
Since $\sigma_1/\sigma_2=\beta_T\to0$,
\begin{align*}
 \frac{(\hat s_{1,T}^{U})^2}{(\hat s_{2,T}^{U})^2} &=\{1+\op(1)\}\frac{\lambda^U_{\bar\sigma}}{1-\lambda^U_{\bar\sigma}}\beta_T^2\pto0,\\
 \left(\frac{L_{\Delta,T}^{V}}{L_{\Delta,T}^{U}}\right)^2 &=\frac{(\hat s_{1,T}^{V})^2+(\hat s_{2,T}^{V})^2} {(\hat s_{1,T}^{U})^2+(\hat s_{2,T}^{U})^2} =\{1+\op(1)\}\lambda^U_{\bar\sigma}.
\end{align*}
Taking square roots gives the second limit in \eqref{eq:length-critical-ratios}. For the optimal mean, relative variance consistency under both policies gives
\begin{align*}
 \frac{L_{1,T}^{V}}{L_{1,T}^{U}} =\frac{\hat\sigma_{1,T}^{V}}{\hat\sigma_{1,T}^{U}} \sqrt{\frac{n_{1,T}^{U}}{n_{1,T}^{V}}} =\{1+\op(1)\} \sqrt{\frac{n_{1,T}^{U}/T}{n_{1,T}^{V}/T}} \pto\infty.
\end{align*}
This proves the first limit in \eqref{eq:length-critical-ratios}.

Finally, the arm-2 count fractions give
\begin{align*}
 \frac{\mathrm{Reg}_T^V}{\mathrm{Reg}_T^U} =\frac{n_{2,T}^V/T}{n_{2,T}^U/T} \pto\frac1{\lambda^U_{\bar\sigma}},
\end{align*}
proving \eqref{eq:regret-critical-ratio}.

To prove \eqref{eq:length-neyman-benchmark}, note that for positive variances and real counts $0<n<T$, Cauchy--Schwarz gives
\begin{align*}
 \frac{\sigma_1^2}{n}+\frac{\sigma_2^2}{T-n} \geq\frac{(\sigma_1+\sigma_2)^2}{T},
\end{align*}
with equality at $n=T\sigma_1/(\sigma_1+\sigma_2)$. Under \eqref{eq:comparison-near-cancellation}, \eqref{eq:length-critical-v-variances} gives $L_{\Delta,T}^V=2z_{1-\alpha/2}\sigma_2\{1+\op(1)\}/\sqrt T$. Since $\sigma_1/\sigma_2\to0$, dividing by the Wald length for this fixed-allocation benchmark gives \eqref{eq:length-neyman-benchmark}.

\subsection{Regret--length relation and coverage}\label{app:comparison-reciprocity}
For both comparisons, the calculations above give
\begin{align*}
 \frac{(\hat\sigma_{1,T}^\pi)^2/n_{1,T}^\pi} {\sigma_2^2/n_{2,T}^\pi}\pto0, \qquad \frac{(\hat\sigma_{2,T}^\pi)^2}{\sigma_2^2}\pto1, \qquad \pi\in\{U,V\}.
\end{align*}
Substitution into \eqref{eq:comparison-lengths} gives \eqref{eq:comparison-arm2-length}. The common variance and nominal-level factors cancel in the ratio of squared lengths, giving
\begin{align*}
 \left(\frac{L_{\Delta,T}^V}{L_{\Delta,T}^U}\right)^2 =\{1+\op(1)\}\frac{n_{2,T}^U}{n_{2,T}^V}.
\end{align*}
Multiplication by the exact regret ratio $n_{2,T}^V/n_{2,T}^U$ proves \eqref{eq:regret-length-reciprocity}.

The effective-sample-size checks above also justify coverage for the intervals being compared. For variance-aware UCB, apply \cref{thm:fixed-target} to $u=(1,0)$ and $(1,-1)$. For canonical UCB, the count approximation \eqref{eq:length-ucb-stability} and the effective-sample-size bounds just established allow us to apply \cref{lem:stable-wald}, with $m_{a,T}=n_{a,T}^{U\star}$. Every interval in Table~\ref{tab:interval-lengths} therefore has limiting coverage $1-\alpha$.

\section{Probability lemmas}\label{app:probability-tools}

To apply probability bounds at counts selected by the policy, we control all counts in the relevant range at once. Unless stated otherwise, we work with a single bounded reward sequence whose law may depend on $T$, leaving this dependence implicit. A union bound gives the corresponding events for the two arms.

\subsection{Maximal inequalities}

We use Kolmogorov's maximal inequality to bound the change in a partial sum when replacing one count by another (\citet[Theorem~2.5.5]{durrett2019probability}).

\begin{lemma}\label{lem:kolmogorov-maximal}
Let $\zeta_1,\ldots,\zeta_m$ be independent, mean-zero random variables with finite variances. Then, for every $x>0$,
\begin{align*}
  \Prob\left(\max_{1\leq k\leq m}\abs{\sum_{i=1}^k\zeta_i}\geq x\right) \leq\frac{1}{x^2}\sum_{i=1}^m\E[\zeta_i^2].
\end{align*}
\end{lemma}

To bound the relative index error at all counts, we use the following maximal form of Bernstein's inequality; we include the proof here for completeness.

\begin{samepage}
\begin{lemma}\label{lem:max-bernstein}
Let $\zeta_1,\ldots,\zeta_m$ be independent, mean-zero random variables with $\abs{\zeta_i}\leq1$ and $\sum_{i=1}^m\E[\zeta_i^2]\leq v$. Then, for every $x>0$,
\begin{align}
  \Prob\left( \max_{1\leq k\leq m} \abs{\sum_{i=1}^k\zeta_i} \geq x \right) \leq 2\exp\left \{-\frac{x^2}{2(v+x/3)}\right\}.\label{eq:max-bernstein}
\end{align}
\end{lemma}
\end{samepage}

\begin{proof}
We may assume $v>0$, since otherwise every $\zeta_i$ is zero almost surely. For $\theta\in(0,3)$, $|\zeta_i|\leq1$ and $r!\geq2\cdot3^{r-2}$ for $r\geq2$ give
\begin{align*}
 \E e^{\theta \zeta_i} \leq1+\E[\zeta_i^2]\sum_{r=2}^{\infty}\frac{\theta^r}{r!}\leq1+\frac{\theta^2\E[\zeta_i^2]}{2(1-\theta/3)} \leq\exp\left\{\frac{\theta^2\E[\zeta_i^2]}{2(1-\theta/3)}\right\}.
\end{align*}
Consequently,
\begin{align*}
  \cM_k := \exp\left\{ \theta\sum_{i=1}^k\zeta_i -\frac{\theta^2}{2(1-\theta/3)} \sum_{i=1}^k\E[\zeta_i^2] \right\}, \qquad k=0,\ldots,m,
\end{align*}
is a nonnegative supermartingale starting at $\cM_0=1$. If $\max_{k\leq m}\sum_{i=1}^k\zeta_i\geq x$, then at the first crossing we have
\begin{align*}
  \cM_k \geq \exp\left\{ \theta x-\frac{\theta^2v}{2(1-\theta/3)} \right\}.
\end{align*}
By Ville's inequality,
\begin{align*}
  \Prob\left( \max_{k\leq m}\sum_{i=1}^k\zeta_i\geq x \right) \leq \exp\left\{ -\theta x+\frac{\theta^2v}{2(1-\theta/3)} \right\}.
\end{align*}
Taking $\theta=x/(v+x/3)\in(0,3)$ makes the exponent $-x^2/[2(v+x/3)]$. We obtain \eqref{eq:max-bernstein} by applying the same argument to $-\zeta_i$ and adding the two tail bounds.
\end{proof}

\subsection{Uniform approximation of the empirical upper confidence bound}

We measure the index error relative to the population bonus, which remains positive even at zero reward variance because it includes the inverse-count term.

Let $X_1,X_2,\ldots$ be i.i.d.\ rewards supported on $[0,1]$, with mean $\mu:=\E[X_1]$ and variance $\sigma^2:=\var(X_1)$, and let $\bar X_n$ and $\hat\sigma_n$ be defined as in \eqref{eq:sample-mean}. Write
\begin{align*}
  B_T(n):=\sigma\sqrt{\frac{c_T}{n}}\vee\frac{c_T}{n}, \qquad \hat B_T(n):=\hat\sigma_n\sqrt{\frac{c_T}{n}}\vee\frac{c_T}{n},
\end{align*}
and let $\ell_T:=\log\log(e^eT)$.

\begin{lemma}\label{lem:uniform-ucbv}
For every $\varepsilon\in(0,1)$,
\begin{align}
  \Prob\bigg( \sup_{1\leq n\leq T} \frac{\abs{\bar X_n+\hat B_T(n)-\mu-B_T(n)}}{B_T(n)} >\varepsilon+\sqrt{2\varepsilon} \bigg)\leq 4(1+\log_2T) \exp\left\{-\frac{3}{14}\varepsilon^2c_T\right\}.\label{eq:uniform-ucbv-tail}
\end{align}
Consequently, there is a deterministic sequence $\eta_T$ satisfying
\begin{gather}
  \eta_T\leq C\left(\frac{\ell_T}{c_T}\right)^{1/4} \to 0,\label{eq:eta-ucbv} \\
  \Prob\bigg( \sup_{1\leq n\leq T} \frac{\abs{\bar X_n+\hat B_T(n)-\mu-B_T(n)}}{B_T(n)} \leq\eta_T \bigg)\to 1.\label{eq:uniform-ucbv}
\end{gather}
Here $C$ is universal. Both statements hold uniformly over reward laws supported on $[0,1]$.
\end{lemma}

\begin{proof}
For the sample mean, write $S_n:=\sum_{i=1}^n(X_i-\mu)$. For $j\geq0$, set $m_j:=2^j$. Whenever $n\in[m_j,2m_j]\cap[1,T]$, we have
\begin{align*}
  nB_T(n) =\sigma\sqrt{c_Tn}\vee c_T \geq \sigma\sqrt{c_Tm_j}\vee c_T,
\end{align*}
so $\abs{S_n}>\varepsilon\,nB_T(n)$ at any count in this block implies
\begin{align*}
  \max_{k\leq2m_j}\abs{S_k} > \varepsilon\left(\sigma\sqrt{c_Tm_j}\vee c_T\right).
\end{align*}
In \cref{lem:max-bernstein}, take $v=2m_j\sigma^2$ and set $w=m_j\sigma^2\vee c_T$, $x=\varepsilon\sqrt{c_Tw}$. Then $v\leq2w$ and $x\leq w$ because $\varepsilon\leq1$ and $c_T\leq w$. Hence
\begin{align*}
 \frac{x^2}{2(v+x/3)} \geq\frac{\varepsilon^2c_Tw}{2(2w+w/3)}=\frac{3}{14}\varepsilon^2c_T.
\end{align*}
Summing over the dyadic blocks gives
\begin{align}
  \Prob\left( \sup_{1\leq n\leq T} \frac{\abs{\bar X_n-\mu}}{B_T(n)}>\varepsilon \right) \leq 2(1+\log_2T) \exp\left\{-\frac{3}{14}\varepsilon^2c_T\right\}.\label{eq:uniform-mean-tail}
\end{align}

To handle the empirical bonus, we repeat the argument with the centered squared rewards $V_i:=(X_i-\mu)^2-\sigma^2$. Since $\abs{V_i}\leq1$ and $\var(V_i)\leq\E[(X_i-\mu)^4]\leq\sigma^2$, we obtain
\begin{align}
  \Prob\left( \sup_{1\leq n\leq T} \frac{\abs{n^{-1}\sum_{i=1}^nV_i}}{B_T(n)}>\varepsilon \right) \leq 2(1+\log_2T) \exp\left\{-\frac{3}{14}\varepsilon^2c_T\right\}.\label{eq:uniform-second-moment-tail}
\end{align}

Outside the two events in \eqref{eq:uniform-mean-tail} and \eqref{eq:uniform-second-moment-tail}, both errors are bounded by $\varepsilon B_T(n)$ for every $n\leq T$. Using $\hat\sigma_n^2-\sigma^2=n^{-1}\sum_{i=1}^nV_i-(\bar X_n-\mu)^2$, we obtain
\begin{align*}
  \abs{\hat\sigma_n^2-\sigma^2} \leq\varepsilon B_T(n)+\varepsilon^2B_T(n)^2.
\end{align*}
Let $w_n:=\sqrt{c_T/n}$. If $n\leq c_T$, then $w_n\geq1$, whereas both standard deviations are at most $1/2$. Hence $\sigma w_n<w_n^2$ and $\hat\sigma_nw_n<w_n^2$, so the inverse-count term determines both bonuses: $\hat B_T(n)=B_T(n)=c_T/n$. If $n>c_T$, we have $w_n<1$ and $B_T(n)\leq1$, so $\varepsilon\leq1$ gives
\begin{align*}
  \abs{\hat\sigma_n^2-\sigma^2} \leq2\varepsilon B_T(n).
\end{align*}
Taking the maximum with a fixed value is one-Lipschitz. Combining this with $\abs{\sqrt{x}-\sqrt{y}}\leq\sqrt{\abs{x-y}}$ and $B_T(n)\geq w_n^2$ gives
\begin{align*}
  \abs{\hat B_T(n)-B_T(n)} \leq w_n\abs{\hat\sigma_n-\sigma}\leq \sqrt{2\varepsilon}\,B_T(n).
\end{align*}
Adding the sample-mean error proves \eqref{eq:uniform-ucbv-tail}.

For the second statement, take $\varepsilon=\varepsilon_T^{\mathrm{idx}}:=c_0\sqrt{\ell_T/c_T}$ with $c_0$ sufficiently large. The bound in \eqref{eq:uniform-ucbv-tail} then tends to zero, while $\eta_T:=\varepsilon_T^{\mathrm{idx}}+\sqrt{2\varepsilon_T^{\mathrm{idx}}}$ satisfies $\eta_T\leq C(\ell_T/c_T)^{1/4}$. This proves \eqref{eq:eta-ucbv}--\eqref{eq:uniform-ucbv}.
\end{proof}

The proof also gives separate mean and bonus bounds. With $\varepsilon_T^{\mathrm{idx}}=c_0\sqrt{\ell_T/c_T}\to0$, let
\begin{align}
 \cE_T^{\rm sep}:=\bigcap_{a=1}^2\left\{ \sup_{n\leq T}\frac{|\bar X_{a,n}-\mu_a|}{B_a(n)}\leq \varepsilon_T^{\mathrm{idx}}, \quad \sup_{n\leq T}\frac{|\hat B_a(n)-B_a(n)|}{B_a(n)}\leq\sqrt{2\varepsilon_T^{\mathrm{idx}}} \right\}.\label{eq:separate-index-errors}
\end{align}
The mean and centered-second-moment tail bounds imply
\begin{align*}
 \Prob\{(\cE_T^{\rm sep})^c\} \leq8(1+\log_2T)e^{-3(\varepsilon_T^{\mathrm{idx}})^2c_T/14}\longrightarrow0
\end{align*}
for sufficiently large fixed $c_0$.

For the allocation proofs, let 
\begin{align}
\cE_T^{\rm unif} := \bigcap_{a=1}^2 \left\{ \sup_{1\leq n\leq T} \frac{\abs{\UCB_a(n)-\mu_a-B_a(n)}}{B_a(n)} \leq \eta_T \right\}. \label{eq:uniform-ucbv-event}
\end{align}
Applying \cref{lem:uniform-ucbv} to each arm and taking a union bound gives
\begin{align}
  \Prob(\cE_T^{\rm unif})\to 1.\label{eq:uniform-event-prob}
\end{align}
On this event, including at counts visited by the policy,
\begin{align}
  \mu_a+(1-\eta_T)B_a(n) \leq \UCB_a(n) \leq \mu_a+(1+\eta_T)B_a(n).\label{eq:ucbv-upper-lower}
\end{align}

\subsection{Moment bounds and deterministic-count approximation}

The bound above applies at every count, but it is not sharp enough for the comparisons near the stability boundary. We first control the mean and standard deviation over all counts beyond a deterministic lower bound.

\begin{lemma}\label{lem:tail-moments}
For the single-arm reward sequence above, write
\begin{align*}
 S_n=\sum_{i=1}^n(X_i-\mu),\qquad S^{(2)}_n=\sum_{i=1}^n\{(X_i-\mu)^2-\sigma^2\}.
\end{align*}
For every integer $k\geq1$ and $x>0$,
\begin{align}
 \Prob\left(\sup_{n\geq k}\frac{|S_n|}{n}>x\right) \vee\Prob\left(\sup_{n\geq k}\frac{|S^{(2)}_n|}{n}>x\right) \leq\frac{4\sigma^2}{kx^2}. \label{eq:random-index-tail-maximal}
\end{align}
Consequently, for deterministic integers $k_T\to\infty$,
\begin{align}
 \sup_{n\geq k_T}|\bar X_n-\mu| &=\sigma\,\Op(k_T^{-1/2}),\qquad \sup_{n\geq k_T}|\hat\sigma_n-\sigma| =\Op(k_T^{-1/2}).\label{eq:tail-mean-sd}
\end{align}
\end{lemma}
\begin{proof}
Let $\zeta_i$ denote either $X_i-\mu$ or $(X_i-\mu)^2-\sigma^2$. In both cases $\E \zeta_i=0$ and $\E \zeta_i^2\leq\sigma^2$, since $\E(X_i-\mu)^4\leq\sigma^2$. By \cref{lem:kolmogorov-maximal},
\begin{align*}
 \Prob\left(\sup_{n\geq k}\frac{|\sum_{i=1}^n\zeta_i|}{n}>x\right) \leq\sum_{j=0}^{\infty} \Prob\left(\max_{n\leq2^{j+1}k}\left|\sum_{i=1}^n\zeta_i\right| >x2^jk\right)\leq\sum_{j=0}^{\infty} \frac{2^{j+1}k\sigma^2}{x^2 2^{2j}k^2} =\frac{4\sigma^2}{kx^2}.
\end{align*}
This proves \eqref{eq:random-index-tail-maximal} and the mean bound. For $\sigma>0$, set $\widetilde\sigma_n^2=n^{-1}\sum_{i=1}^n(X_i-\mu)^2$. Then
\begin{align*}
 |\widetilde\sigma_n-\sigma| &\leq\frac{|S^{(2)}_n|}{n\sigma},\\
 |\hat\sigma_n-\widetilde\sigma_n| &\leq|\bar X_n-\mu|, \qquad\widetilde\sigma_n^2 =\hat\sigma_n^2+(\bar X_n-\mu)^2.
\end{align*}
Taking suprema over $n\geq k_T$ and using \eqref{eq:random-index-tail-maximal} gives
\begin{align*}
 \sup_{n\geq k_T}|\hat\sigma_n-\sigma| \leq\frac1\sigma\sup_{n\geq k_T}\frac{|S^{(2)}_n|}{n} +\sup_{n\geq k_T}|\bar X_n-\mu| =\Op(k_T^{-1/2}).
\end{align*}
If $\sigma=0$, both estimation errors vanish almost surely.
\end{proof}

\begin{lemma}\label{lem:deterministic-count}
For the i.i.d.\ reward sequence and bonuses defined above, let $n_T\to\infty$ be deterministic positive integers. Then
\begin{align}
  \bar X_{n_T}+\hat B_T(n_T) = \mu+B_T(n_T) + \Op\left(\frac{B_T(n_T)}{\sqrt{c_T}}\right).\label{eq:deterministic-count}
\end{align}
\end{lemma}

\begin{proof}
By \cref{lem:tail-moments} with $k_T=n_T$,
\begin{align*}
 \bar X_{n_T}-\mu=\sigma\,\Op(n_T^{-1/2}),\qquad \hat\sigma_{n_T}-\sigma=\Op(n_T^{-1/2}).
\end{align*}
Taking the maximum with a fixed value is one-Lipschitz, so
\begin{align*}
 |\hat B_T(n_T)-B_T(n_T)| \leq\sqrt{c_T/n_T}\,|\hat\sigma_{n_T}-\sigma| =\Op(\sqrt{c_T}/n_T).
\end{align*}
Adding the two errors proves \eqref{eq:deterministic-count}, since $B_T(n_T)/\sqrt{c_T}=\sigma/\sqrt{n_T}\vee\sqrt{c_T}/n_T$.
\end{proof}

\subsection{Consistency of the empirical standard deviation at a random sample size}

For studentization, the standard-deviation error must be small relative to the population standard deviation, which may tend to zero. The next lemma gives this relative bound under a diverging effective sample size, even when the count depends on the observations.

\begin{samepage}
\begin{lemma}\label{lem:random-index-variance}
For each $T$, let $X_{T,1},X_{T,2},\ldots$ be i.i.d.\ and supported on $[0,1]$, with mean $\mu_T$ and variance $\sigma_T^2>0$. Let $\tilde N_T$ be a positive integer-valued random variable, possibly dependent on the reward sequence and on additional randomness. Suppose
\begin{align}
  \tilde N_T\sigma_T^2\xrightarrow{p}\infty.\label{eq:random-index-effective-size}
\end{align}
Then
\begin{align*}
  \frac{\hat\sigma_{T,\tilde N_T}}{\sigma_T}\xrightarrow{p}1.
\end{align*}
In particular, \eqref{eq:random-index-effective-size} holds whenever $\tilde N_T/m_T\xrightarrow{p}1$ and $m_T\sigma_T^2\to\infty$ for a positive deterministic sequence $m_T$.
\end{lemma}
\end{samepage}

\begin{proof}
By \eqref{eq:random-index-effective-size}, choose a deterministic sequence $\omega_T\to\infty$ such that
\begin{align}
 \Prob(\tilde N_T\sigma_T^2\geq\omega_T)\to1. \label{eq:random-index-diagonal}
\end{align}
Set $k_T=\floor{\omega_T/\sigma_T^2}\vee1$. Then $k_T\sigma_T^2\to\infty$ and $\Prob(\tilde N_T\geq k_T)\to1$. Write $S_{T,n}$ and $S^{(2)}_{T,n}$ for the centered sums in \cref{lem:tail-moments}. Its proof gives
\begin{align*}
 \left|\frac{\hat\sigma_{T,n}}{\sigma_T}-1\right| \leq\frac{|S^{(2)}_{T,n}|}{n\sigma_T^2} +\frac{|S_{T,n}|}{n\sigma_T}.
\end{align*}
For every $\eps>0$, \eqref{eq:random-index-tail-maximal} therefore yields
\begin{align*}
 \Prob\left(\left|\frac{\hat\sigma_{T,\tilde N_T}}{\sigma_T}-1\right|>\eps\right)\leq\Prob(\tilde N_T<k_T) +\frac{16}{\eps^2 k_T\sigma_T^2} +\frac{16}{\eps^2 k_T}\longrightarrow0.
\end{align*}
This bound does not require $\tilde N_T$ to be independent of the rewards.
\end{proof}

\subsection{Martingale Gaussian approximation}

The next lemma states the martingale CLT and a consequence for variables measurable at time zero. The latter is used in the cross-arm argument.

\begin{lemma}\label{lem:martingale-clt}
Let $(d_{T,t},\cG_{T,t})_{t=1}^{m_T}$ be a square-integrable martingale difference array with deterministic row lengths $m_T$. Suppose that, for every $\eps>0$,
\begin{align*}
 \sum_{t=1}^{m_T}\E[d_{T,t}^2\mid\cG_{T,t-1}]&\pto1,\\
 \sum_{t=1}^{m_T} \E[d_{T,t}^2\one\{|d_{T,t}|>\eps\}\mid\cG_{T,t-1}]&\pto0.
\end{align*}
Then $\mathcal Z_T:=\sum_{t=1}^{m_T}d_{T,t}\dto N(0,1)$. Moreover, for every fixed $v\in\R$ and every sequence of complex, $\cG_{T,0}$-measurable variables $\Upsilon_T$ with $|\Upsilon_T|\leq1$,
\begin{align*}
 \E[\Upsilon_Te^{iv\mathcal Z_T}]-\E[\Upsilon_T]e^{-v^2/2}\longrightarrow0.
\end{align*}
\end{lemma}
\begin{proof}
The Gaussian convergence follows from the martingale CLT \citep[Chapter~3]{hall1980martingale}. For the second assertion, first take real $\Upsilon_T$ and define
\begin{align*}
 \frac{d\mathbb Q_T}{d\Prob} =\frac{2+\Upsilon_T}{2+\E \Upsilon_T}.
\end{align*}
Write this density as $\mathcal D_T$. It is $\cG_{T,0}$-measurable and $1/3\leq \mathcal D_T\leq3$. For integrable $V$ and any event $E$,
\begin{align*}
 \E_{\mathbb Q_T}[V\mid\cG_{T,t-1}] &=\frac{\E[\mathcal D_TV\mid\cG_{T,t-1}]}{\mathcal D_T} =\E[V\mid\cG_{T,t-1}],\\
 \mathbb Q_T(E)&\leq3\Prob(E).
\end{align*}
Thus the conditional moments are unchanged, and both assumptions continue to hold in $\mathbb Q_T$-probability. The same martingale CLT gives $\E_{\mathbb Q_T}e^{iv\mathcal Z_T}\to c_v:=e^{-v^2/2}$. Hence
\begin{align*}
 \E[\Upsilon_Te^{iv\mathcal Z_T}]-\E[\Upsilon_T]c_v =(2+\E \Upsilon_T)(\E_{\mathbb Q_T}e^{iv\mathcal Z_T}-c_v)-2(\E e^{iv\mathcal Z_T}-c_v)\longrightarrow0.
\end{align*}
Apply this argument to the real and imaginary parts for complex $\Upsilon_T$.
\end{proof}

\subsection{Replacing a random count by a deterministic count}

Write $S_n=\sum_{i=1}^n(X_i-\mu)$ for the centered sum from one arm. The next lemma gives an Anscombe-type random-index replacement \citep{anscombe1952large}. We bound the partial-sum increments throughout a short window, allowing the random count to depend on the observations.

\begin{lemma}\label{lem:deterministic-window}
Let $k_T\to\infty$ be deterministic integers and let $\tilde N_T$ be random positive integers, possibly dependent on the observations. If $\tilde N_T/k_T\pto1$, then
\begin{align}
 \frac{S_{\tilde N_T}-S_{k_T}}{\sigma\sqrt{k_T}}\pto0,\label{eq:deterministic-window-increments}
\end{align}
provided $\sigma>0$.
\end{lemma}
\begin{proof}
Fix an integer $k\geq4$, $0<\delta<1/2$, and $\eps>0$. Let $k_\delta=\ceil{\delta k}<k$. For any positive integer-valued $\tilde N$, on $\{|\tilde N/k-1|\leq\delta\}$ we have
\begin{align*}
 |S_{\tilde N}-S_k| \leq\max_{0\leq j\leq k_\delta}|S_{k+j}-S_k| \vee\max_{0\leq j\leq k_\delta}|S_k-S_{k-j}|.
\end{align*}
Applying \cref{lem:kolmogorov-maximal} to the forward and backward blocks gives
\begin{align}
 \Prob\left(|S_{\tilde N}-S_k|>\eps\sigma\sqrt{k}\right) \leq\Prob(|\tilde N/k-1|>\delta)+\frac{2\ceil{\delta k}}{\eps^2k}.\label{eq:finite-window-bound}
\end{align}
Take $k=k_T$ and $\tilde N=\tilde N_T$. Letting $T\to\infty$ and then $\delta\downarrow0$ proves \eqref{eq:deterministic-window-increments}.
\end{proof}

\subsection{Gaussian approximation under stable counts}\label{app:stable-wald}

\begin{lemma}\label{lem:stable-wald}
For each $T$, consider two independent reward sequences, each i.i.d.\ and supported on $[0,1]$. Let $n_{1,T}$ and $n_{2,T}$ be the terminal counts of a nonanticipating policy that selects one arm at each round and pulls each arm at least once. The policy may use independent randomization. Form $W_T(u)$ and $D_T$ as in \eqref{eq:linear-wald} and \eqref{eq:uniform-error}. Suppose there are positive deterministic counts $m_{a,T}$ such that
\begin{align*}
 n_{a,T}/m_{a,T}\pto1,\qquad m_{a,T}\sigma_a^2\to\infty, \qquad a=1,2.
\end{align*}
Then $D_T\to0$.
\end{lemma}
\begin{proof}
Let $\cF_{T,t}$ now also include any independent policy randomization at time zero. Then $A_t$ is $\cF_{T,t-1}$-measurable. For any deterministic $u=u_T\neq0$, define
\begin{align*}
 s_T^{\star}(u)^2 &=\sum_{a=1}^2\frac{u_a^2\sigma_a^2}{m_{a,T}}, \qquad e_{a,T}(u) =\frac{u_a\sigma_a}{\sqrt{m_{a,T}}s_T^{\star}(u)},\\
 d_{T,t}(u) &=\frac{u_{A_t}(R_t-\mu_{A_t})} {m_{A_t,T}s_T^{\star}(u)}, \qquad W^{\star}_T(u) =\sum_{t=1}^T d_{T,t}(u) =\frac{1}{s_T^{\star}(u)} \sum_{a=1}^2\frac{u_aS_{a,n_{a,T}}}{m_{a,T}}.
\end{align*}
Independence of each new reward from the past gives
\begin{align*}
 \E[d_{T,t}(u)\mid\cF_{T,t-1}]&=0,\\
 \E[d_{T,t}(u)^2\mid\cF_{T,t-1}] &=\frac{u_{A_t}^2\sigma_{A_t}^2}{m_{A_t,T}^2s_T^{\star}(u)^2}.
\end{align*}
Since $\sum_a e_{a,T}(u)^2=1$,
\begin{align*}
 \left|\sum_{t=1}^T \E[d_{T,t}(u)^2\mid\cF_{T,t-1}]-1\right| =\left|\sum_{a=1}^2e_{a,T}(u)^2 \left(\frac{n_{a,T}}{m_{a,T}}-1\right)\right|\leq\max_a\left|\frac{n_{a,T}}{m_{a,T}}-1\right|\pto0.
\end{align*}
Also, bounded rewards imply
\begin{align*}
 \max_{t\leq T}|d_{T,t}(u)| \leq\max_a\frac{|u_a|}{m_{a,T}s_T^{\star}(u)} \leq\max_a\frac1{\sqrt{m_{a,T}\sigma_a^2}}\longrightarrow0.
\end{align*}
The deterministic upper bound gives the conditional Lindeberg condition. Thus \cref{lem:martingale-clt} yields $W^{\star}_T(u_T)\dto N(0,1)$ for every deterministic $u_T\neq0$.

Set $\tilde S_{a,T}=S_{a,n_{a,T}}/(\sigma_a\sqrt{m_{a,T}})$ and $\tilde S_T=(\tilde S_{1,T},\tilde S_{2,T})$. Taking $u=(1,0)$ and $(0,1)$ above gives $\|\tilde S_T\|_2=\Op(1)$. By stability and \cref{lem:random-index-variance},
\begin{align*}
 \eps_T&:=\max_a\left|\frac{m_{a,T}}{n_{a,T}}-1\right|\pto0,\\
 \delta_T&:=\max_a\left| \frac{m_{a,T}}{n_{a,T}} \frac{\hat\sigma_{a,n_{a,T}}^2}{\sigma_a^2}-1\right|\pto0.
\end{align*}
Since $W^{\star}_T(u)=\sum_ae_{a,T}(u)\tilde S_{a,T}$,
\begin{align*}
 \frac{\hat\theta_T(u)-\theta_T(u)}{s_T^{\star}(u)}-W^{\star}_T(u) &=\sum_a e_{a,T}(u)\tilde S_{a,T} \left(\frac{m_{a,T}}{n_{a,T}}-1\right),\\
 \frac{\hat s_T(u)^2}{s_T^{\star}(u)^2}-1 &=\sum_a e_{a,T}(u)^2 \left(\frac{m_{a,T}}{n_{a,T}} \frac{\hat\sigma_{a,n_{a,T}}^2}{\sigma_a^2}-1\right).
\end{align*}
Cauchy--Schwarz and $\sum_a e_{a,T}(u)^2=1$ now give
\begin{align*}
 \sup_{u\neq0}\left| \frac{\hat\theta_T(u)-\theta_T(u)}{s_T^{\star}(u)}-W^{\star}_T(u) \right| &\leq\eps_T\|\tilde S_T\|_2,\\
 \sup_{u\neq0}\left|\frac{\hat s_T(u)^2}{s_T^{\star}(u)^2}-1\right| &\leq\delta_T.
\end{align*}
On $\{\delta_T\leq1/2\}$, both empirical variances are positive, and these bounds give
\begin{align*}
 \sup_{u\neq0}|W_T(u)-W^{\star}_T(u)| \leq C(\eps_T+\delta_T)\|\tilde S_T\|_2=\op(1).
\end{align*}
The event has probability tending to one, so $W_T(u_T)\dto N(0,1)$ for every deterministic $u_T\neq0$. Since $\Phi$ is continuous, convergence of the distribution functions is uniform in $z$. If $D_T\not\to0$, we could choose a subsequence and deterministic $u_T\neq0$ whose Gaussian approximation error stays bounded away from zero, contradicting this convergence. Hence $D_T\to0$.
\end{proof}

\end{document}